\documentclass{article}

\usepackage[preprint,nonatbib]{neurips_2024}
\usepackage[numbers,compress]{natbib}
\usepackage[utf8]{inputenc} 
\usepackage[T1]{fontenc}    
\usepackage{hyperref}       
\usepackage{url}            
\usepackage{booktabs}       
\usepackage{amsfonts}       
\usepackage{nicefrac}       
\usepackage{microtype}      
\usepackage{xcolor}         
\usepackage{preamble}
\usepackage{algorithm}
\usepackage{algpseudocode}
\usepackage{subcaption}

\title{Concentration Bounds for Optimized Certainty Equivalent Risk Estimation}

\author{
Ayon Ghosh \\
Department of Computer Science and Engineering\\
Indian Institute of Technology Madras\\
\texttt{cs21b013@smail.iitm.ac.in} 
 \And
Prashanth L.A.\\
Department of Computer Science and Engineering\\
Indian Institute of Technology Madras\\
\texttt{prashla@cse.iitm.ac.in} 
 \AND
Krishna Jagannathan \\
Department of Electrical Engineering\\
Indian Institute of Technology Madras\\
\texttt{krishnaj@ee.iitm.ac.in} 
}

\begin{document}

\maketitle              

\begin{abstract}
We consider the problem of estimating the Optimized Certainty Equivalent (OCE) risk from independent and identically distributed (i.i.d.) samples. For the classic sample average approximation (SAA) of OCE, we derive mean-squared error as well as concentration bounds (assuming sub-Gaussianity). Further, we analyze an efficient stochastic approximation-based OCE estimator, and derive finite sample bounds for the same. To show the applicability of our bounds, we consider a risk-aware bandit problem, with OCE as the risk. For this problem, we derive bound on the probability of mis-identification. Finally, we conduct numerical experiments to validate the theoretical findings.
\end{abstract}

\section{Introduction}

A major consideration in financial and clinical applications is the quantification of the risk associated with future random outcomes. In the pioneering framework of Markovitz, the variance of the underlying random variable is considered as a measure of risk, and the objective is to maximize the expected value subject to a constraint on the risk, or to maximize an affine function of the mean and variance \cite{markowitz1952portfolio}. Subsequently, another landmark paper \cite{artzner1999coherent} proposed an axiomatic approach to risk, by imposing four properties to be satisfied by the risk measure. 

Specifically, for a r.v. $X$ and a real valued function $\rho$, \cite{artzner1999coherent} imposes the properties of being positively homogeneous, sub-additive, translation invariant and monotone for a risk measure $\rho(X)$. A risk measure that satisfies all these properties is called a \textit{coherent} risk measure. In \cite{follmer2002convex}, the authors relaxed the conditions of sub-additivity and positive homogeneity, and replaced them with a weaker property of convexity, i.e.,
\begin{equation}
    \rho(\lambda X + (1-\lambda)Y) \leq \lambda\rho(X) + (1-\lambda)\rho(Y), \quad \forall \lambda \in [0,1],
    \label{eq:convexity}
\end{equation}
and such a risk measure is said to be \textit{convex}. 

Optimized Certainty Equivalent (OCE) is a family of convex risk measures, introduced  in \cite{ben1986expected} 
 and explored further in \cite{bental2007oce}. Qualitatively, OCE is a class of risk measures which serves as an indicator of the investor's risk appetite, particularly capturing the optimal allocation of resources/losses between the present and the future. Suppose the total loss that an investor could incur in the future is a random variable $X$. The investor has the option to allocate a fraction of their uncertain losses, denoted by $\xi$, to the present. Consequently the present value of losses under a disutility function $\phi$ then becomes $\xi + \EE{\phi(X - \xi)}$. The optimal $\xi$ i.e. the infimum over all possible allocations then gives us the OCE risk of the investor. We shall denote this optimal $\xi$ by $e^{*}$ in this paper. 

Under suitable choices of the disutility function $\phi$, OCE risk encompasses a wide range of risk measures, one of which is the widely used risk measure Conditional Value at Risk (CVaR)\cite{rockafellar2000optimization}. Given a random variable $X$ and a level $\alpha \in (0,1)$, define the VaR as $v_{\alpha}(X) = \inf_{\xi} \left\{ \mathbb{P}[X \leq \xi] \geq \alpha \right\}$. Then the CVaR  $c_{\alpha}(X) = v_{\alpha}(X) + \frac{1}{1-\alpha}\mathbb{E}[X -v_{\alpha}(X)]^{+}$. It can be shown that this is just a special case of OCE risk measure with the disutility function $u(X) =  \frac{1}{1-\alpha}(X)^{+}.$, and $e^{*}$ replaced with $v_{\alpha}(X)$.

\paragraph{Our contributions.}

In this paper, we consider the problem of estimating the OCE risk from independent and identically distributed (i.i.d.) samples of the loss distribution. First, we consider a straightforward Sample Average Estimator for the OCE risk, and bound its mean-squared error. Next, we derive a concentration bound for the SAA estimator when the underlying loss distribution is sub-Gaussian, and the disutility function is strongly convex and smooth. This OCE risk concentration bound enjoys a sub-exponential decay. We illustrate the applicability of our concentration bound in a multi-armed bandit setting with a `Best-OCE-arm' identification problem.

Next, we consider a `streaming' setting wherein the samples from the underlying loss distribution are available one-at-a-time. For this scenario, the sample average estimator is ineffective, because it requires recomputation for each sample. On the other hand, an iterative estimator that updates one step at a time upon receiving each fresh sample is better suited for such a streaming setting. We propose a stochastic approximation-based estimation procedure for OCE risk estimation. We then derive finite-sample bounds for the mean-squared error of the iterative OCE risk estimation procedure.

\paragraph{Related work.}
In recent years, estimation of risk measures from i.i.d. samples has received a lot of research attention, cf. \cite{kagrecha2019distribution,pandey2021estimation,thomas2019concentration,dunkel2010stochastic,bhat2019concentration,prashanth2019concentration,prashanth2016cumulative,wang2010deviation,brown2007large,williamson2020cvarConc,lee2020oce}. The majority of prior works consider CVaR estimation and derive concentration bounds usually under a sub-Gaussianity assumption. 

For OCE risk estimation, bounds are available in \cite{brown2007large,prashanth2022wasserstein}. The former reference considers distributions with bounded support, while the latter includes sub-Gaussian and sub-exponential distributions. In \cite{hamm2013stochastic}, the authors study a stochastic approximation-based procedure for OCE risk estimation.
In comparison to these closely related works, we remark the following:
(i) Unlike \cite{brown2007large}, our OCE risk estimation bounds are for unbounded albeit sub-Gaussian distributions;
(ii) In \cite{prashanth2022wasserstein}, the authors employ a Wasserstein distance based approach for deriving OCE risk concentration bounds for Lipschitz disutility functions. Our present work allows for smooth disutility functions, which cover mean-variance risk measure as an important special case. Moreover, our proof is direct, while they use a unified approach to handle several risk measures that satisfy a continuity criterion. The flip side with their approach is that the constants in the concentration bounds are conservative. Additionally, we derive a concentration bound for the OCE risk minimizer;
(iii) In \cite{hamm2013stochastic}, the authors provide an asymptotic rate result for their OCE risk estimator, while we quantify the rate of convergence through bounds in expectation as well as high probability, in the non-asymptotic regime.

\section{OCE Risk Measure}

\begin{definition}
    \label{def:OCE}
For a random variable (r.v.) $X$, the OCE risk $\oce(X)$ with a disutility function $\phi$ is defined by
    \begin{align}
        \oce(X) := \inf_{\xi} \left\{\xi + \EE{\phi(X - \xi)}\right\}.\label{eq:OCE-def}
    \end{align}
    We denote the minimizer of OCE risk by $e^{*}$. 
\end{definition}

As mentioned earlier, OCE risk accommodates a wide class of popular risk measures, for suitably chosen disutility fucntion $\phi.$ Some common examples for the disutility function and their corresponding OCE risk expressions are given in Table \ref{tab:OCE_risks}.
\begin{table}[htbp]
\centering
\caption{Examples of OCE risks and their corresponding minimizers.}
\label{tab:OCE_risks}
\begin{tabular}{@{}c|c|c|c@{}}
\toprule
\textbf{Risk measure} & \textbf{Disutility function} & \textbf{OCE risk minimizer} & \textbf{OCE risk} \\ 
\midrule
Expected loss & \( \varphi(t) = t \) & \( e^{*} = x, \quad \forall x \in \mathbb{R} \)  & \( \oce(X) = \mathbb{E}[X] \) \\[1ex] 
Entropic risk & \( \varphi_{\gamma}(t) = \frac{1}{\gamma} (e^{\gamma t} - 1) \) & \( e^{*} = \frac{\ln(\mathbb{E}[e^{\gamma X}])}{\gamma} \) & \( \oce(X) = \frac{\ln(\mathbb{E}[e^{\gamma X}])}{\gamma} \)  \\[1.5ex] 
Mean-variance & \( \varphi_c(t) = t + ct^2 \) & \( e^{*} = \mathbb{E}[X] \) & \( \oce(X) = \mathbb{E}[X] + c(\text{Var}(X)) \) \\[1ex] 
Conditional & \multirow{2}{*}{\( \varphi_{\alpha}(t) = \frac{1}{1-\alpha}[t]^+ \)} & \multirow{2}{*}{\( e^{*} = v_{\alpha}(X) \)} & \multirow{2}{*}{\( \oce(X) = c_{\alpha}(X) \)} \\ 
Value-at-Risk &&&\\
\bottomrule
\end{tabular}
\end{table}

\begin{proposition}
If the disutility function \( \phi : \mathbb{R} \rightarrow \mathbb{R}^{+} \cup \{0\} \), which is nondecreasing, closed, and convex satisfies \( \phi(0) = 0 \) and $\phi'(0) = 1$ (see \cite[Defn. 2.1]{bental2007oce}), then the OCE risk, $\oce(X)$ is a convex risk measure (see \cite{bental2007oce} for derivations).
\label{prop:oce_convex_conditions}
\end{proposition}

It can be shown (See ~\cite[Theorem 2.1]{bental2007oce}) that the OCE risk measure under the conditions of \Cref{prop:oce_convex_conditions}  satisfies the following properties, in addition to convexity:
(a) Translation invariance: \( \oce(X + c) = \oce(X) + c, \forall c \in \mathbb{R} \);
(b) Consistency: \( \oce(c) = c \), for any constant \( c \in \mathbb{R} \); and
(c) Monotonicity: Let \( Y \) be any random variable such that \( X(\omega) \leq Y(\omega), \forall \omega \in \Omega \). Then, \( \oce(X) \leq \oce(Y) \).

\section{OCE Risk Estimation and Concentration}
\paragraph{OCE Risk Estimation.}
Let $\{X_1,\ldots,X_n\}$ denote $n$ i.i.d. samples from the distribution $F$ of $X$. Using these samples, we estimate $\oce(X)$ as follows:
\begin{align}
    \label{def:OCE_n}
    \ocen := \inf_{\xi} \left\{ \xi + \sum_{i=1}^{n} \frac{\phi(X_{i} - \xi)}{n} \right\}.
\end{align}
We denote the minimizer of \eqref{def:OCE_n} by $\hat{e}_{n}$. For example, in the case of CVaR estimation as discussed in \cite{prashanth2020concentration}, $\hat{e}_{n}$ is just $\hat{v}_{n,\alpha} = X_{[\lceil(1-\alpha)\rceil]}$, with $\{X_{[i]}\}_{i=1}^{n}$ being the order-statistics for $\{X_i\}_{i=1}^{n}$.

For the $n$-sample, the empirical distribution function (EDF) is defined 
by 
\begin{equation}
    F_n(x) = \frac{1}{n}\sum_{i=1}^{n} \mathbb{I}\{X_{i} \leq x\},\,\, \forall x.
\end{equation}
The OCE risk estimator defined above can be seen as the OCE risk applied to a r.v., say $Z_n$, with distribution $F_n$, i.e.,  $\ocen = oce(Z_{n})$.

\paragraph{Useful expressions for $e^*$ and $\hat e_n$.}
We differentiate \eqref{eq:OCE-def} and using the fact that $e^*$ is the OCE risk minimizer, we obtain 
\begin{align}
\frac{d}{d\xi} \Bigl( \xi + \mathbb{E}[\phi(X-\xi)] \Bigr) = 1 + \mathbb{E}\left[\frac{d}{d\xi}\bigl(\phi(X-\xi)\bigr)\right] = 0 \Rightarrow \mathbb{E}[\phi'(X-e^{*})] = 1.\label{eq:phi-diff}
\end{align}
where the second equality follows from DCT (see \cite[Theorem 4.6.3]{durrett2019probability}) via \Cref{ass:DCT}. 

Let $\hat{e}_{n}$ be the infimum for $\ocen$. Then, by arguments similar to those leading to \eqref{eq:phi-diff}, we obtain
  \begin{equation}
    \frac{1}{n} \sum_{i=1}^{n} \phi'(X_i - \hat{e}_{n}) = 1.\label{eq:phin-diff}
  \end{equation}
For deriving the OCE risk estimation bounds, we require assumptions on the tail of the underlying distribution. Two popular tail assumptions are sub-Gaussian and sub-exponential, which are formalized below.
\begin{definition}[\textbf{Sub-Gaussian distribution}]
    A r.v. $X$ with mean $\mu = \mathbb{E}[X]$ is sub-Gaussian if there is a positive parameter $\sigma$ such that
$
\mathbb{E}[e^{\lambda(X-\mu)}] \leq e^{\frac{\sigma^2 \lambda^2}{2}}, \,\, \forall \lambda. 
$
\end{definition}

\begin{definition}[\textbf{Sub-exponential distribution}]
    A r.v. $X$ with mean $\mu = \mathbb{E}[X]$ is sub-exponential if there are non negative parameters $(\nu, b)$ such that
$
\mathbb{E}[e^{\lambda(X-\mu)}] \leq e^{\frac{\nu^2 \lambda^2}{2}}, \,\, \forall |\lambda| < \frac{1}{b}.
$
\end{definition}

\paragraph{Bounds for OCE risk estimation.}
For deriving mean-squared error (MSE) and concentration bounds for OCE risk estimation, we make the following assumptions.
\begin{assumption}
\label{ass:phi-stronglyconvex}
    The function $\phi$ is $\mu$-strongly convex.
\end{assumption}

\begin{assumption}
\label{ass:phi-smooth}
    The function $\phi$ is $L$-smooth, i.e., 
$
\left| \phi(y) - \left( \phi(x) + \phi'(x)(y - x) \right) \right| \leq \frac{L}{2} (x - y)^2, \forall x,y \in \R.
$
 
\end{assumption}

\begin{assumption}
\label{ass:DCT}
    The function $\phi'$ is continuously differentiable, and the collection of random variables $\{\phi''(X - t) : t \in \mathbb{R}\}$ is uniformly integrable.
\end{assumption}

\begin{assumption}
\label{ass:gaussian}
    The r.v. $X$ is sub-Gaussian
     with parameter $\sigma$.
\end{assumption}
We now comment on the assumptions made above.
\Cref{ass:phi-stronglyconvex} is required for providing MSE/concentration bounds for estimation of OCE risk minimizer $e^*$. If the function $\phi$ is convex but not strongly convex, then it is difficult to bound $\hat e_n - e^*$, since the function could have a arbitrarily wide plateau around $e^*$. For the special case of VaR, such an assumption has been made for obtaining concentration bounds in \cite{prashanth2019concentration}. In this case \Cref{ass:phi-stronglyconvex} translates to a strictly increasing distribution, to ensure a VaR estimate can concentrate as the distribution is not flat around VaR.

In \cite{prashanth2022wasserstein}, the authors assume the disutility function $\phi$ is Lipschitz for deriving a concentration bound for OCE risk estimation. Such an assumption is restrictive since it disallows a mean-variance risk measure via OCE risk, as in some example above.  In contrast, \Cref{ass:phi-smooth} imposes a smoothness condition on $\phi$, in turn allowing a risk measure with mean-variance tradeoff.

\Cref{ass:DCT} is a technical assumption that ensure certain integrals are finite, allowing an application of the dominated convergence theorem to interchange expectation and differentiation operators. For the case of r.v.s with unbounded support, such an assumption has been made earlier in the context of another risk measure, see \cite{gupte2023optimization}.

An assumption on the tail of the underlying distribution is usually made for deriving concentration bounds, cf. \cite{prashanth2019concentration} for CVaR. In \Cref{ass:gaussian}, we impose a sub-Gaussianity requirement on the underlying distribution.


\todop{Discuss how L-smoothness is the "breaking point" for sub-exponentiality and how anything beyond it will not hold and MGF integral will diverge. }

\paragraph{Bounds for estimation of OCE risk minimizer $e^*$.}
The first result that we present is a mean-squared error bound on the estimator $\hat e_n$ of the OCE risk minimizer $e^*$.
\begin{theorem}
\label{thm:mse-estar}
    Suppose \Crefrange{ass:phi-stronglyconvex}{ass:DCT} hold. Further, assume $X$ has a finite second moment. Then, we have the following bound for $\hat{e}_{n}$, which is the minimizer of the empirical OCE risk defined in \eqref{def:OCE_n}: 
\begin{equation}
    \EE{(\hat{e}_{n} - e^{*})^2} \leq \frac{ L^2\left( (e^{*})^2 + \EE{X^2}\right) - 2e^{*}\EE{X} }{n\mu^2},
\end{equation}
where $L$ is the smoothness parameter and $\mu$ is the strong-convexity parameter of the disutility function $\phi$, and $e^*$ is the OCE risk minimizer.
\end{theorem}
\begin{proof}
See Appendix \ref{proof:thm:mse-estar}
\end{proof}
The mean-squared error bound in the result above is useful in bounding the error in OCE estimation owing to the following relation:
\[|\ocen - \oce(X)| 
\leq \frac{3L\Bigl(\hat{e}_{n} - e^{*} \Bigr)^{2}}{2} + \frac{ \sum_{i=1}^{n} \left|\phi(X_{i} - e^{*}) - E [ \phi(X_{i} - e^{*}) ]\right| }{n}. \]

Under an additional sub-Gaussianity assumption, we present below a concentration bound for estimation of OCE risk minimizer $e^*$.
\begin{theorem}
\label{thm:estar-conc}
Suppose \Crefrange{ass:phi-stronglyconvex}{ass:gaussian} hold. Then we have
\begin{align}
    P\biggl[ \bigl| \hat{e}_{n}- e^{*}\bigr| \geq \epsilon \biggr] \leq 2\expo{-\frac{n\mu^2\epsilon^2}{8L^2\sigma^2}},\label{eq:estar-concentration}  
\end{align}
where the quantities $L,e^*,\mu$ are as defined in \Cref{thm:mse-estar}, while $\sigma^2$ is the sub-Gaussianity parameter of $X$.
\end{theorem}
\begin{proof}
    See Appendix \ref{proof:thm:estar-conc}.
\end{proof}
From the result above, it is apparent that the bound exhibits a Gaussian tail decay.

\paragraph{Bounds for OCE risk estimation.}
For our first result which is a mean squared error bound for OCE risk, we require an additional assumption stated below.
\begin{assumption}
    \label{ass:b_combined}
    $\EE{(\phi(X-e^{*}))^p}$ is bounded for $p = 1,2$ and $\EE{(\phi'(X-e^{*}))^k}$ is bounded for $k = 2,3,4$.
\end{assumption}
The moment bounds can be seen to be easily satisfied for a sub-Gaussian/sub-exponential $X$. In the general case, such moment bounds are necessary owing to the fact that the mean-squared error derivation naturally leads to terms involving $\EE{Z^2}$, $\EE{Y^2}$ and $\EE{Y^4}$, where 
$Z = \phi(X-e^{*}) - \mathbb{E}[\phi(X-e^{*})]$ and $Y = \phi'(X-e^{*}) - \mathbb{E}[\phi'(X-e^{*})] = \phi'(X-e^{*}) - 1$.

The result below bounds the mean-squared error of the OCE risk estimator defined in \eqref{def:OCE_n}.
\begin{theorem}
    \label{thm:mse-oce}
    Suppose \ref{ass:phi-stronglyconvex} to \ref{ass:DCT} and \ref{ass:b_combined} hold. Let $\var(X)$ denote the variance of a r.v. $X$. Then, we have
    \begin{align*}
    \EE{[\ocen - \oce(X)]^2} \!\leq \!\frac{2}{n}\var(\phi(X-e^{*})) \! + \frac{27L^2[\var{(\phi'(X-e^{*}))}]^2}{2n^2\mu^4} \! + \frac{9L^2\EE{(\phi'(X-e^{*}))^4}}{2n^3\mu^4}.
    \end{align*}
\end{theorem}
\begin{proof}
See Appendix \ref{proof:thm:mse-oce}.
\end{proof}
Using Jensen's inequality, we can infer that
\[\EE{|\ocen - \oce(X)|} \le \sqrt{\EE{[\ocen - \oce(X)]^2}} \le \frac{K_1}{\sqrt{n}},\]
where $K_1$ can be inferred from the result above.

For OCE risk estimates to concentrate, we require a bound on the tail probability $P \Bigr( \Bigl| \phi(X - e^{*}) - \mathbb{E} [ \phi(X - e^{*}) ]\Bigr| \geq \epsilon \Bigl)$.
The lemma below establishes that $\phi(X - e^{*}) - \mathbb{E} [ \phi(X - e^{*}) ]$ is sub-exponential.
\begin{lemma}
\label{lem:lsmooth-subg}
Suppose $X$ is $\sigma^2$-sub-Gaussian and $\phi : \mathbb{R} \to \mathbb{R}^{+}$ is closed, convex, L-smooth, non-decreasing and satisfies $\phi(0) = 0$ and $\phi'(0)= 1$. Then the zero-mean r.v. $\phi(X-e^{*}) - \mathbb{E}[\phi(X-e^{*})]$ is sub-exponential with parameters $(\frac{4C_{1}}{c_{2}},\frac{2}{c_{2}})$, where $C_{1} = 2( 4 + \expo{3c_{0}| \frac{L(e^{*})^2}{2} - e^{*}|} - 3c_{0}|\frac{L(e^{*})^2}{2} - e^{*}| )$, $c_{2} = \frac{c_{0}}{2}$ and $c_{0} = \min\left(\frac{1}{12L\sigma^{2}}, \frac{1}{12\sigma|Le^{*}-1|}\right)$.
\end{lemma}

\begin{proof}
    See Appendix \ref{proof:lem:lsmooth-subg}.
\end{proof}

We can now use this result to provide a Bernstein-type concentration bound for the OCE risk estimate in \eqref{def:OCE_n}.

\begin{theorem}
\label{thm:oce-conc1}
Suppose \Crefrange{ass:phi-stronglyconvex}{ass:gaussian} hold. Let $c_{0} = \min\left(\frac{1}{12L\sigma^{2}}, \frac{1}{12|Le^{*}-1|\sigma}\right)$, $C_{1} = 2( 4 + \expo{3c_{0}| \frac{L(e^{*})^2}{2} - e^{*}|} - 3c_{0}|\frac{L(e^{*})^2}{2} - e^{*}|)$ , $c_{2} = \frac{c_{0}}{2}$. Then, for any $\epsilon>0$, we have 
\begin{align}
\Prob{ \left| \ocen - \oce(X) \right| > \epsilon } \leq 2\expo{ \frac{-c_{2}n\epsilon^{2}}{4(4C_{1}+\epsilon)}} + 2\expo{-\frac{\mu^{2}n\epsilon}{24L^{3}\sigma^{2}}}.\label{eq:tb1}
\end{align}
\end{theorem}
\begin{proof}
    See Appendix \ref{proof:thm:oce-conc1}.
\end{proof}
\begin{remark}
We can see that the tail bound stated above exhibits an exponential tail decay. Since we assume that the function $\phi$ is strongly convex, a tighter sub-Gaussian decay does not hold. Intuitively, the r.v. $\phi(X - e^{*})$ underlying OCE risk concentration is bounded below by a quadratic function of $X$, which precludes sub-Gaussian concentration for OCE risk.
\end{remark}
\begin{remark}
    In \cite{prashanth2022wasserstein}, the authors assume $\phi$ is Lipschitz and employ a Wasserstein distance-based approach to arrive at a bound with a sub-Gaussian tail. In contrast, the bound in  \Cref{thm:oce-conc1} exhibits sub-exponential tail decay for a $L$-smooth $\phi$. From \Cref{tab:OCE_risks}, it is apparent that our bounds are applicable for the mean-variance risk measure, since the underlying function $\phi$  is smooth.
\end{remark}
We can invert the bound in \ref{thm:oce-conc1} to arrive at the following  `high-confidence' form:
\begin{corollary}
\label{cor:oceest}
Under conditions of \Cref{thm:oce-conc1}, for any  $\delta \in (0,1)$, with probability at least $(1-\delta)$, we have
\begin{align*}
        \left| \ocen - \oce(X) \right| &\leq \left[\frac{1}{{c_2n}} + \frac{{6L^3\sigma^2}}{{\mu^2n}} \right] \log\frac{2}{\delta} + \sqrt{\left[\frac{1}{{c_2}} + \frac{{6L^3\sigma^2}}{{n\mu^2}} \right]^2 \log^2\left(\frac{2}{\delta}\right) + \frac{{8C_1}}{{c_2n}} \log\frac{2}{\delta}}.
\end{align*}
\end{corollary}
While the tail bound in \eqref{eq:tb1} is useful for the bandit application under best arm identification framework in Section \ref{sec:bandits}, the equivalent high-confidence form above cannot be employed for a upper confidence bound type algorithm to minimize regret, since $c_0, C_1$ require the knowledge of $e^*$, which is not known in a typical bandit setting.


\section{Stochastic approximation for OCE risk estimation}
\label{sec:sto-approx}
\paragraph{The estimators.}

Recall that the OCE risk minimizer $e^*$ satisfies $\mathbb{E}[\phi'(X-e^{*})] - 1 =0$. Using stochastic approximation, we arrive at the following update iteration for obtaining an estimate $t_j$ of the OCE risk minimizer $e^*$:
    \begin{align}
 t_j = t_{j-1} - \gamma_j (1-\phi'(X_{j}-t_{j-1})), \label{eq:sto-approx}  
    \end{align}
where $\{X_{j}\}$ are i.i.d. samples from the distribution $X$ and $\gamma_j$ is a suitably chosen step size. Compared to the batch estimator in the previous section, the above update is more amenable to `streaming' settings, where the samples arrive one at a time.

Inspired by \cite{moulines2011long}, we derive bounds for the \textit{averaged iterate}, and use this quantity to estimate OCE. These quantities, denoted by  $\bar{t}_m$ and $oce_{m,sa}^{\phi}$ are defined as follows:
 \begin{align}
\bar{t}_m = \frac{1}{m}\sum_{i=0}^{m-1} t_{i}\,\, \textrm{ and } \,\,\ocesa = \bar{t}_{m} +  \frac{\sum_{i=1}^{m} \phi(X_{i} - \bar{t}_{m})}{m}.\label{eq:oce-sgd}
 \end{align}

\paragraph{Results.}
For deriving a mean-squared error bound, we require the following assumption in addition to \Crefrange{ass:phi-stronglyconvex}{ass:DCT} and \ref{ass:b_combined}:
\begin{assumption}
    \label{ass:b1}
     The second derivative $\phi''$ is $M$-lipschitz. That is for all $x_{1}$,$x_{2}$ $\in \mathbb{R} $, $|\phi''(x_{1}) - \phi''(x_{2})| \leq M|x_{1}-x_{2}|$.
\end{assumption}

The main result that provides a mean-squared error bound for the averaged iterate $\bar t_m$ after $m$ iterations of \eqref{eq:sto-approx} is given below.

\begin{theorem}
    \label{thm:bach-t_star}
 Suppose \Crefrange{ass:phi-stronglyconvex}{ass:DCT}, \ref{ass:b_combined} and \ref{ass:b1} hold. 
 Suppose $m$ iterations of \eqref{eq:sto-approx} are run with stepsize $\gamma_{j} = \frac{b}{j^{\alpha}}$ with $\alpha \in (0.5,1)$.
 Then, the averaged iterate $\bar t_m$ satisfies
\begin{align}
 \mathbb{E} &[(\bar{t}_{m}-e^{*})^{2}] \leq \frac{\K^2}{m},  \label{eq:tm-sa-bound}
\end{align}
where 
\begin{align*}
 \K &= \frac{\sigma}{\mu} + \frac{6 \sigma}{\mu b^{1 / 2}} + \frac{M b \tau^{2}}{2 \mu^{3 / 2}}\left(1+(\mu b)^{1 / 2}\right) + \frac{4 L b^{1 / 2}}{\mu} \\
 &\quad+ \frac{8 A}{\mu^{1 / 2}}\left(\frac{1}{b}+L\right)\left( \mathbb{E}[(t_{0}-e^{*})^{2}] +\frac{\sigma^{2}}{L^{2}}\right)^{1 / 2} + \frac{5 M b^{1 / 2} \tau}{2 \mu} A \exp \left(24 L^{4} b^{4}\right)\\
 &\quad\times\left( \mathbb{E}[(t_{0}-e^{*})^{2}] +\frac{\mu \mathbb{E}\left[(t_{0}-e^{*})^{4}\right]}{20 b \tau^{2}}+2 \tau^{2} b^{3} \mu+8 \tau^{2} b^{2}\right)^{1 / 2},
\end{align*}
and $A$ is a constant that depends only on $\mu, b, L$ and $\alpha$.
\begin{proof}
The proof follows by an application of \cite[Theorem 3]{moulines2011long} and verifying the requisite conditions there. The reader is referred to Appendix \ref{proof:thm:bach-t_star} for the details.
\end{proof}
\end{theorem}
The mean-squared error bound from \eqref{eq:tm-sa-bound} is comparable to the corresponding result for the batch estimator \eqref{eq:phin-diff} in \Cref{thm:mse-estar}, if we set $m=n$, i.e., run \eqref{eq:sto-approx} for $n$-samples and compare it to $\hat e_n$.

Next, we use the  bound in \eqref{eq:tm-sa-bound} to derive a bound on the mean-squared error for the OCE risk estimate in \eqref{eq:oce-sgd}.
\begin{theorem}
\label{thm:exp_oce}
Under conditions of \Cref{thm:bach-t_star}, we have
\begin{equation}
\EE{|\ocesa(X) - \oce(X)|} \leq \frac{L \K^2}{2m} + \frac{\K\sqrt{\var{(\phi'(X-e^{*}))}}}{m} + \frac{\var{(\phi(X-e^{*}))}}{\sqrt{m}}.
\end{equation}
where $\K$ is as defined in \Cref{thm:bach-t_star}.
\end{theorem}
\begin{proof}
    See Appendix \ref{proof:thm:exp_oce}.
\end{proof}


\section{Application: Multi-Armed Bandits}
\label{sec:bandits}
We consider a a $K$-armed stochastic bandit problem, which is characterized by the probability distributions of the arms, denoted as $P_1,\ldots,P_K$. We focus on identifying the arm exhibiting the lowest OCE risk value within a predetermined sampling budget. Here, a bandit algorithm engages with the environment over a fixed budget comprising $n$ rounds. At each round $t = 1,\ldots,n$, the algorithm selects an arm $I_t \in {1,\ldots,K}$ and records a cost sample from the distribution $P_{I_t}$. Upon completing the $n$ rounds, the bandit algorithm suggests an arm $J_n$, and is evaluated based on the probability of mis-identifying the optimal arm, denoted as $P[J_n \neq i^{\ast}]$, where $i^{\ast}$ represents the arm with the lowest OCE risk value, i.e., $i^* = \argmin_{i=1,\ldots,K} oce_{i}^{\phi}$.

Algorithm 1 outlines the pseudocode for our OCE-SR algorithm, tailored to identify the OCE-optimal arm under a fixed budget constraint. This algorithm presents a modification of the conventional successive rejects (SR) approach, with a distinction: while regular SR employs sample means to estimate the expected value of each arm, OCE-SR utilizes empirical OCE risk, as described in \Cref{def:OCE_n}, for estimating the OCE risk of each arm. The elimination strategy, involving $K-1$ phases and discarding the arm with the poorest OCE risk estimate at the conclusion of each phase, is borrowed from the regular SR framework. 

\begin{algorithm}
\caption{OCE-based SR Algorithm}\label{oce_algo}
\begin{algorithmic}[1]
\State \textbf{Initialization:} Set $A_1 = \{1, \ldots, K\}$, $\overline{\log K} = \frac{1}{2} + \sum_{i=2}^{K} \frac{1}{i}$, $n_0 = 0$, $n_k = \left\lceil \frac{\left(n - K \right)}{\log K(K+1-k)} \right\rceil$, for $k = 1, \ldots, K - 1$.
\For{$k = 1, 2, \ldots, K - 1$}
    \State Play each arm in $A_k$ for $t_{k} = (n_k - n_{k-1})$ times.
    \State Compute the OCE risk infimum estimates $\hat{e}_{t_{k}}^{i}$ by solving \Cref{eq:phin-diff} for each arm $i$ in $A_k$.
    \State Compute the OCE risk estimate $\hat{oce}_{t_{k}}^{i}$ for each arm $i \in A_k$ using \Cref{def:OCE_n} for each arm $i$ in $A_k$.
    \State Set $A_{k+1} = A_k \setminus \arg\max_{i \in A_k} \hat{oce}_{t_k}^{i}$, i.e., remove the arm with the highest empirical OCE risk, with ties broken arbitrarily.
\EndFor
\State \textbf{Output:} Return the solitary element in $A_K$.
\end{algorithmic}
\end{algorithm}
The ensuing result delves into the performance analysis of the OCE-SR algorithm for sub-Gaussian distributions.
\begin{theorem}
\label{thm:oce-algo}
Consider a $K$-armed stochastic bandit, where the arms follow a sub-Gaussian distribution. Let arm-$[i]$ to denote the arm with the $i^{th}$ lowest OCE risk value. Let $\Delta_{[i]} = \oce_{i} - \oce_{i^{*}}$ represent the difference between the OCE risk values of $arm-[i]$ and the optimal arm. For a given budget $n$, the arm, say $J_{n}$, returned by the OCE-SR algorithm satisfies:
\[
\mathbb{P}\left[J_{n} \neq i^{*}\right] \leq 4 K(K-1) \exp \left(-\frac{(n-K)(1-\alpha) G_{\max }}{H \overline{\log } K}\right),
\]
where $G_{\max }$ is a problem dependent constant that does not depend on the underlying OCE risk gaps $\Delta_{[i]}$ and $n$, and
$
H=\max _{i \in\{1,2 \ldots, K\}} \frac{i}{\min \left\{\Delta_{[i]} / 2, \Delta_{[i]}^{2} / 4\right\}}.
$

\begin{proof}
    See Appendix \ref{proof:thm:oce-algo}.
\end{proof} 

\end{theorem}


\section{Simulation experiments}
\label{sec:expts}
In this section, we illustrate the effectiveness of our proposed OCE estimators in \Cref{def:OCE_n} and \Cref{eq:oce-sgd} on two different settings. In the first setting, we investigate the performance of our OCE estimators in a synthetic experimental setup. In the second setting, we apply our stochastic approximation-based OCE estimator \Cref{eq:oce-sgd} to the credit risk model studied earlier in \cite{dunkel2010stochastic,hegde2021ubsr}.

\paragraph{Synthetic Setup.}
For this experiment, we consider a normal distribution with mean $0.5$ and variance $25$.
We set the disutility function $\phi(t) = t + \frac{t^2}{2}$, and this choice satisfies the smoothness assumption \Cref{ass:phi-smooth}. From \Cref{tab:OCE_risks}, we know $e^{*} = 0.5$ and $oce^{\phi}(X) = 13$. Figure \ref{fig:oce-batch-normal} presents the estimation errors $|\hat{e}_{n} - e^{*}|$ and $|oce_{n}^{\phi}-oce^{\phi}(X)|$ as a function of the number of samples ($n$). The results  are averages over $1000$ independent replications.
From Figure \ref{fig:oce-batch-normal}, it is apparent that the estimators \Cref{eq:phin-diff} and \Cref{def:OCE_n} converge rapidly to the true values.

\begin{figure}[htbp]
    \centering
    \begin{subfigure}[b]{0.40\textwidth}
        \centering
        \includegraphics[width=\textwidth]{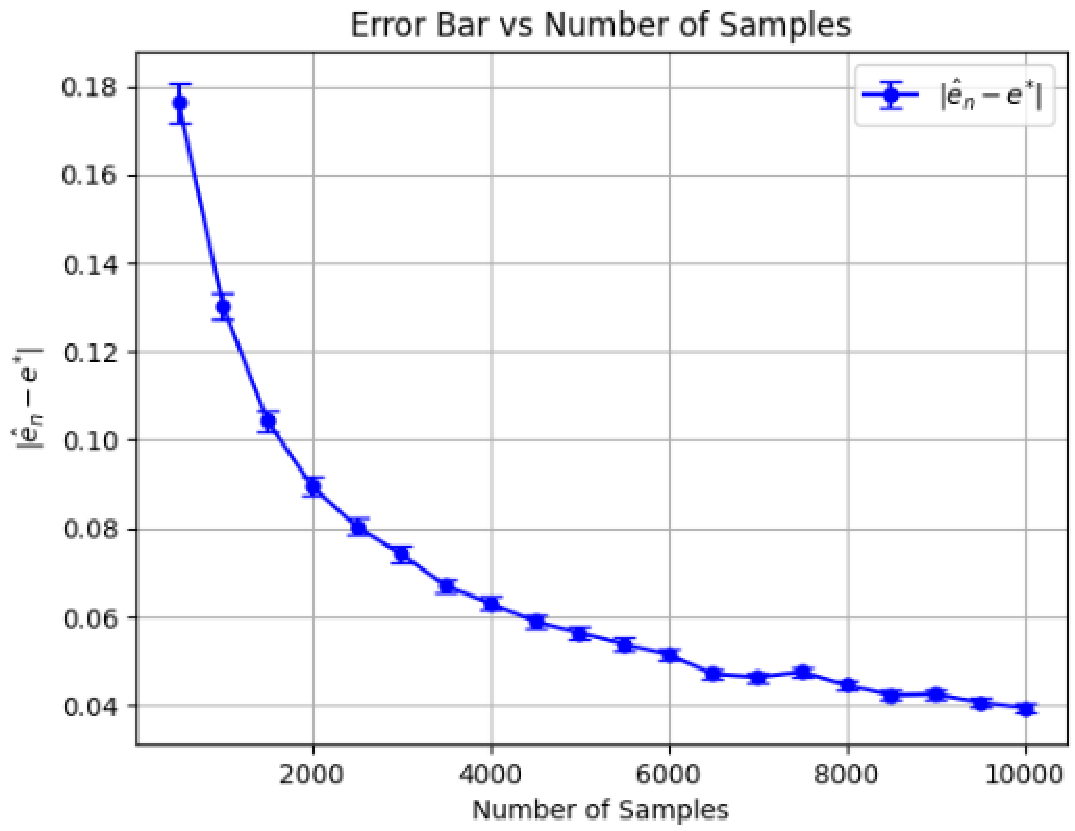}
        \caption{The error in estimation of OCE risk minimizer as a function of the number of samples.}
        \label{fig:subfig1}
    \end{subfigure}
    \hfill
    \begin{subfigure}[b]{0.40\textwidth}
        \centering
        \includegraphics[width=\textwidth]{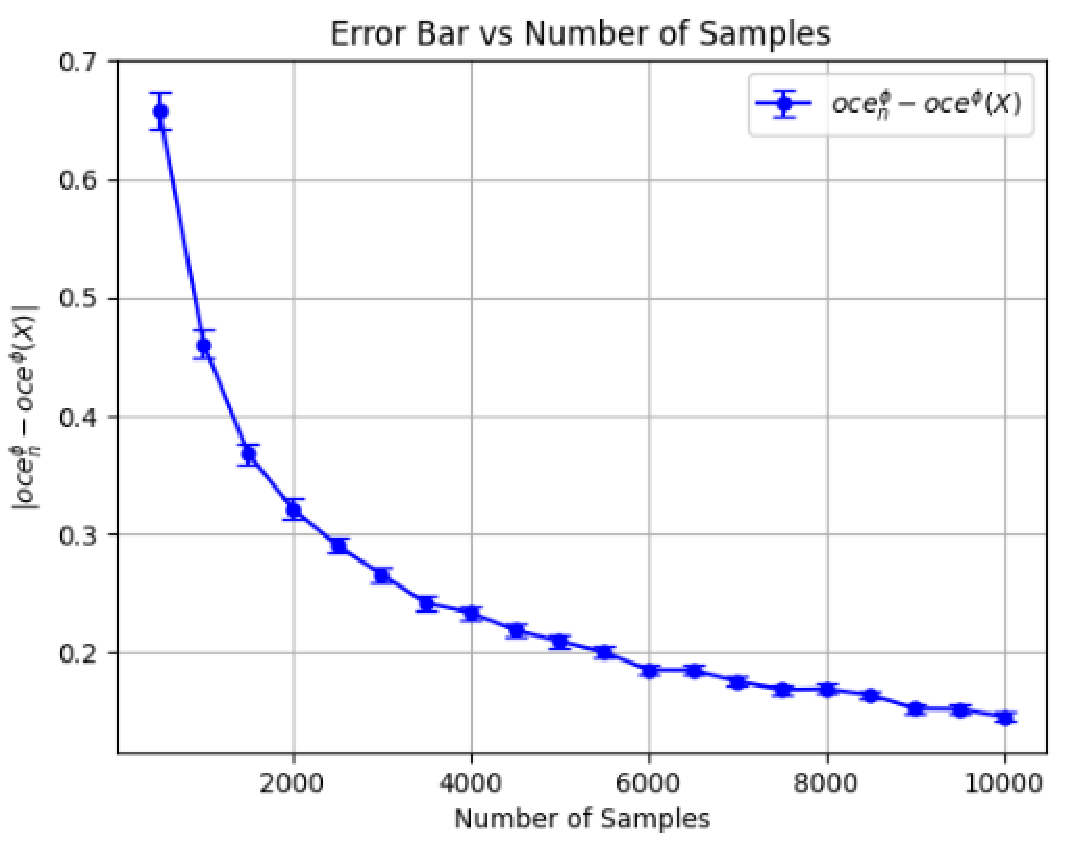}
        \caption{The error in estimation of OCE risk as a function of the number of samples.}
        \label{fig:subfig2}
    \end{subfigure}
    \caption{Errors in estimation of OCE risk and its minimizer, when the underlying distribution is $\mathcal{N}(0.5,5^2)$. OCE risk and its minimizer are estimated using \eqref{def:OCE_n} and \eqref{eq:phin-diff}. The results are averages over $1000$ independent replications.}
    \label{fig:oce-batch-normal}
\end{figure}

Next, we present results for the streaming estimator described in Section \ref{sec:sto-approx}.
We set the disutility function as $\phi(t) = t + \frac{t^2}{2}$. From \Cref{tab:OCE_risks}, $e^{*} = 0.5$ and $oce^{\phi}(X) = 13$. We carried out our stochastic approximation scheme \Cref{eq:sto-approx} for $5000$ iterations and replicated the experiment $1000$ times independently and took the averaged results for different step sizes. The plots for $\EE{(\bar{t}_{k}-e^{*})^2}$ as a function of the number of samples ($k$) are in \Cref{fig:subfig3} and \Cref{fig:subfig4}. Figure \ref{fig:streaming} demonstrates a clear and swift convergence of the estimators in \ref{eq:oce-sgd} towards their true values.

\begin{figure}[htbp]
    \centering
    \begin{subfigure}[b]{0.45\textwidth}
        \centering
        \includegraphics[height=1.7in]{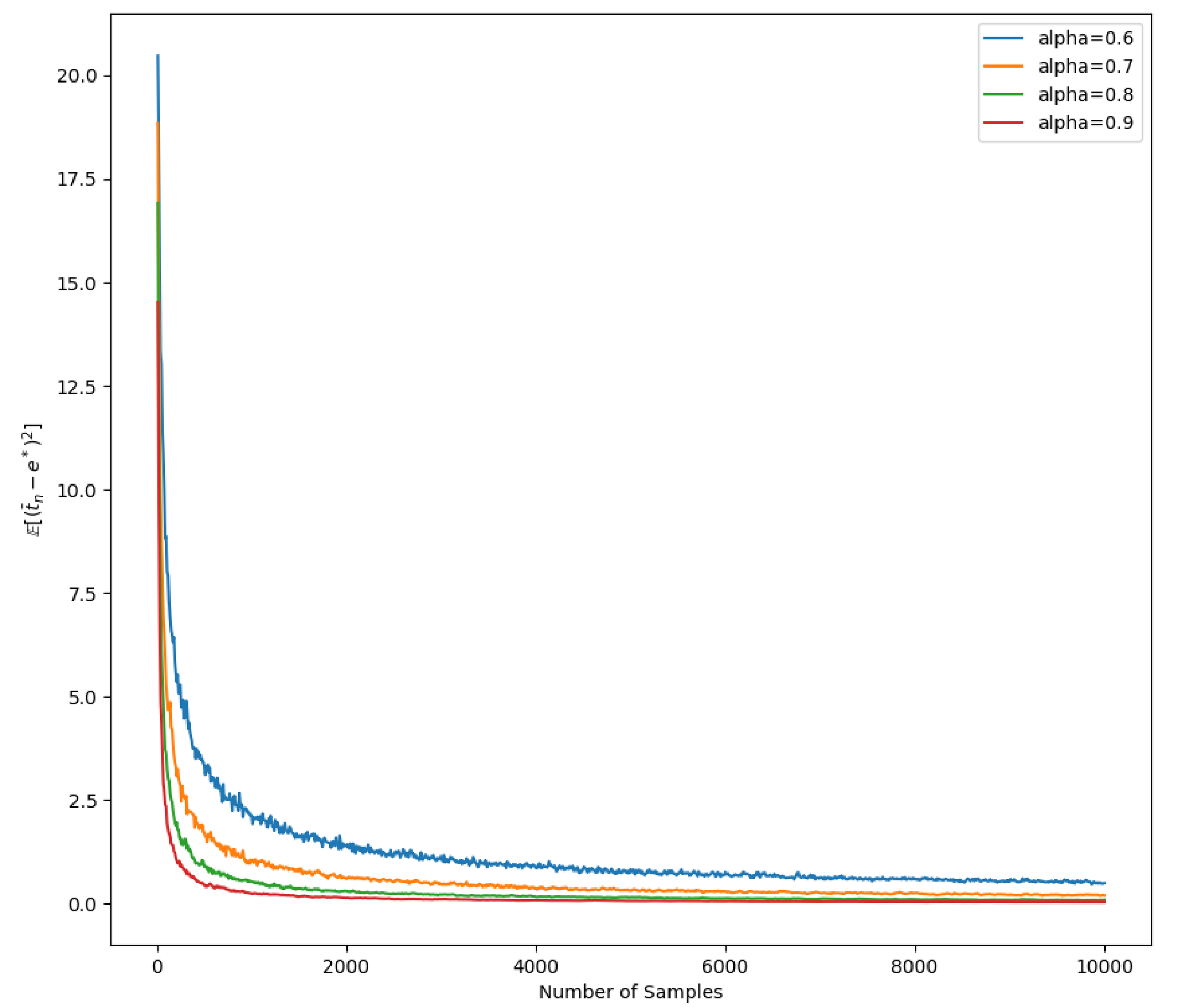}
        \caption{$\mathbb{E}[(\bar{t}_{k}-e^{*})^2]$}
        \label{fig:subfig3}
    \end{subfigure}
    \begin{subfigure}[b]{0.45\textwidth}
        \centering
        \includegraphics[height=1.7in]{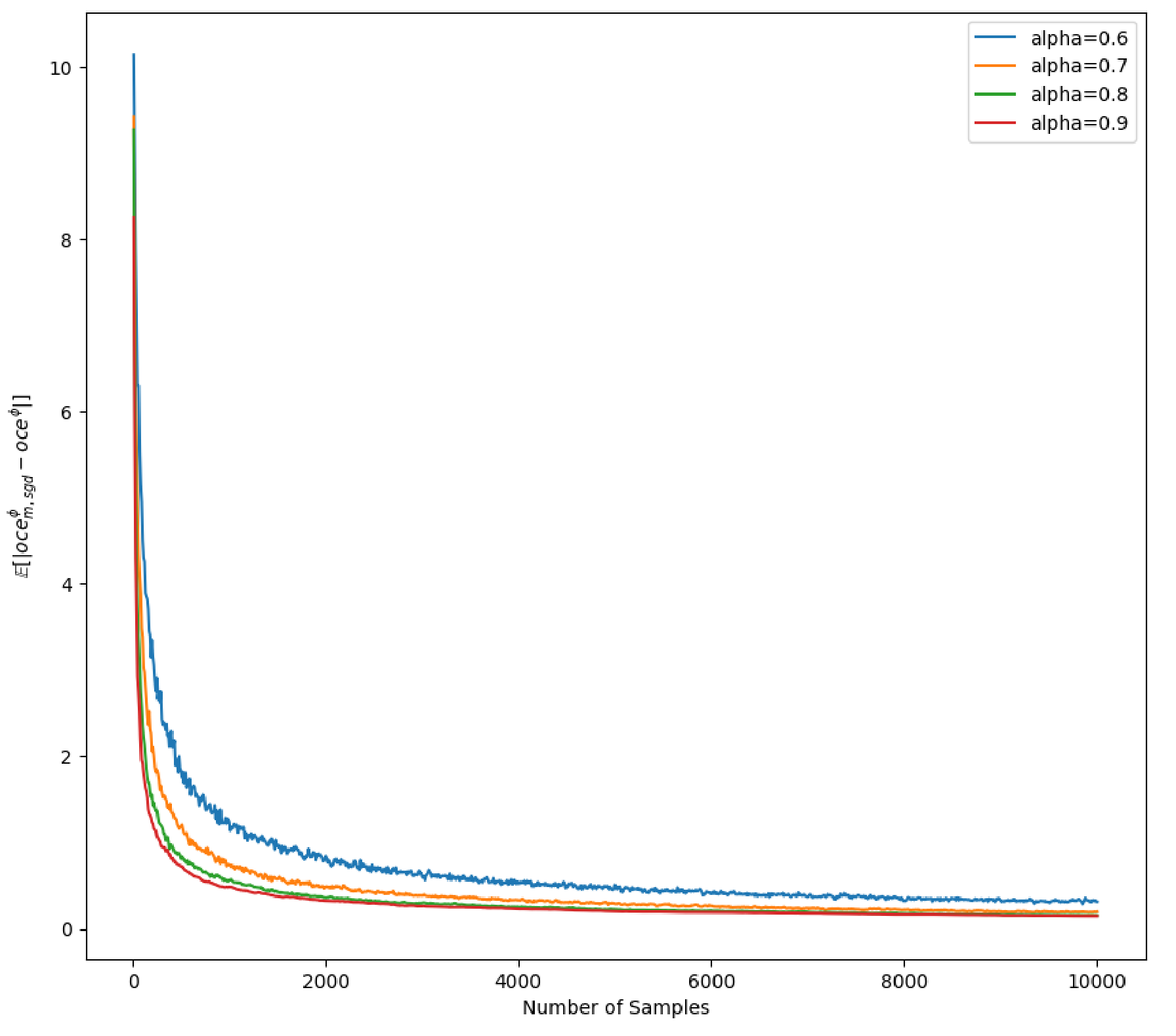}
        \caption{$\mathbb{E}[|oce^{\phi}_{k,\text{sgd}}-oce^{\phi}(L)|]$}
        \label{fig:subfig4}
    \end{subfigure}
    \caption{Errors in estimation of OCE risk and its minimizer, when the underlying distribution is $\mathcal{N}(0.5,5^2)$. OCE risk minimizer is estimated using \eqref{eq:sto-approx}, while OCE risk is estimated using \eqref{eq:oce-sgd}. The step size $\gamma_j= \frac{10}{j^{\alpha}}$ and $t_{0} = 1$. The results are averages over $1000$ independent replications.}
    \label{fig:streaming}
\end{figure}

\paragraph{Credit Risk Model.}

In this experiment, we follow the credit risk model, which is described next. Suppose an investor 's portfolio has $m$ positions, with each position subject to some risk of defaulting causing loss to the investor. The total loss is $L = \sum_{i=1}^{m} v_{i}D_{i}$, with $D_{i}$ being an indicator variable which is $1$ if the $i^{th}$ position defaults and $0$ otherwise and $v_{i} > 0$ is the fractional loss associated with the $i^{th}$ position. In order to quantify this, let $D_{i} = \mathbb{I}\{R_{i} > r_{i}\}$ where $r_{i} = \Phi^{-1}(1-p_{i})$ and $p_{i} = 0.05$ are the threshold risk and marginal default probability of the $i^{th}$ position. Moreover, $R_{i}$, the r.v. corresponding to the defaulting risk of the $i^{th}$ position is determined by the following factor model:
 For $i = 1,...,m, d < m$, 
\begin{equation*}
    R_{i} = A_{i,0}\epsilon_{i} + \sum_{j=1}^{d} A_{i,j}Z_{j}
\textrm{ with }
A_{i,0}^2 + ... + A_{i,d}^2 = 1, \, A_{i,0} > 0, A_{i,j} \geq 0.
\end{equation*}
Here, $Z_{1},...,Z_{d}$ are the systematic risk variables and $\epsilon_{1},...,\epsilon_{d}$ are the idiosyncratic risk variables and all of them are assumed to be distributed as $\mathbb{N}(0,1)$. The parameters $A_{i,j}$ denote the cross-coupling coefficients. For testing our OCE estimator on this model, we decided to use the setup described in \cite{dunkel2010stochastic}: Let the number of positions $m = 25$, fractional losses $v_1 =...= v_5 = 1.00, v_6 =...= v_{10} =
1.25, v_{11} =...= v_{15} = 1.50, v_{16} =...= v_{20} = 1.75,
v_{21} =...= v_{25} = 2.00$. $p_{i} = 0.05$ for all positions.The coupling parameters $A_{i,j}$ are given as $A_{1,1}=...=A_{5,1}=0.1$, $A_{6,2}=...=A_{10,2}=0.1$, $A_{11,3}=...=A_{15,3} = 0.1$, $A_{16,4}=...=A_{20,4}=0.1$, $A_{21,5}=...=A_{25,5}=0.1$, $A_{i,6} = 0.1$ and $A_{i,j} = 0$ otherwise. The distutility function we choose was $\phi(t) = t + \frac{t^2}{2}$. Under these values, it is easy to see that from \Cref{tab:OCE_risks}, that $e^{*} = 1.875$, and $oce^{\phi}(L) = 3.28515625$. 

\Cref{fig:subfig5} and \Cref{fig:subfig6} presents the plots of $\EE{(\bar{t}_{k}-e^{*})^2}$ and $\EE{|oce^{\phi}_{k,sgd}-oce^{\phi}(L)|}$ (respectively) as a function of the number of samples ($k$). The reported result represent the average of $1000$ independent replications. From \label{fig:cred_risk}, we can see the rapid progress of our estimators in \eqref{eq:oce-sgd} to their true values under the credit risk model setup.
\begin{figure}[htbp]
    \centering
    \begin{subfigure}{0.45\textwidth}
        \centering
        \includegraphics[height=1.7in]{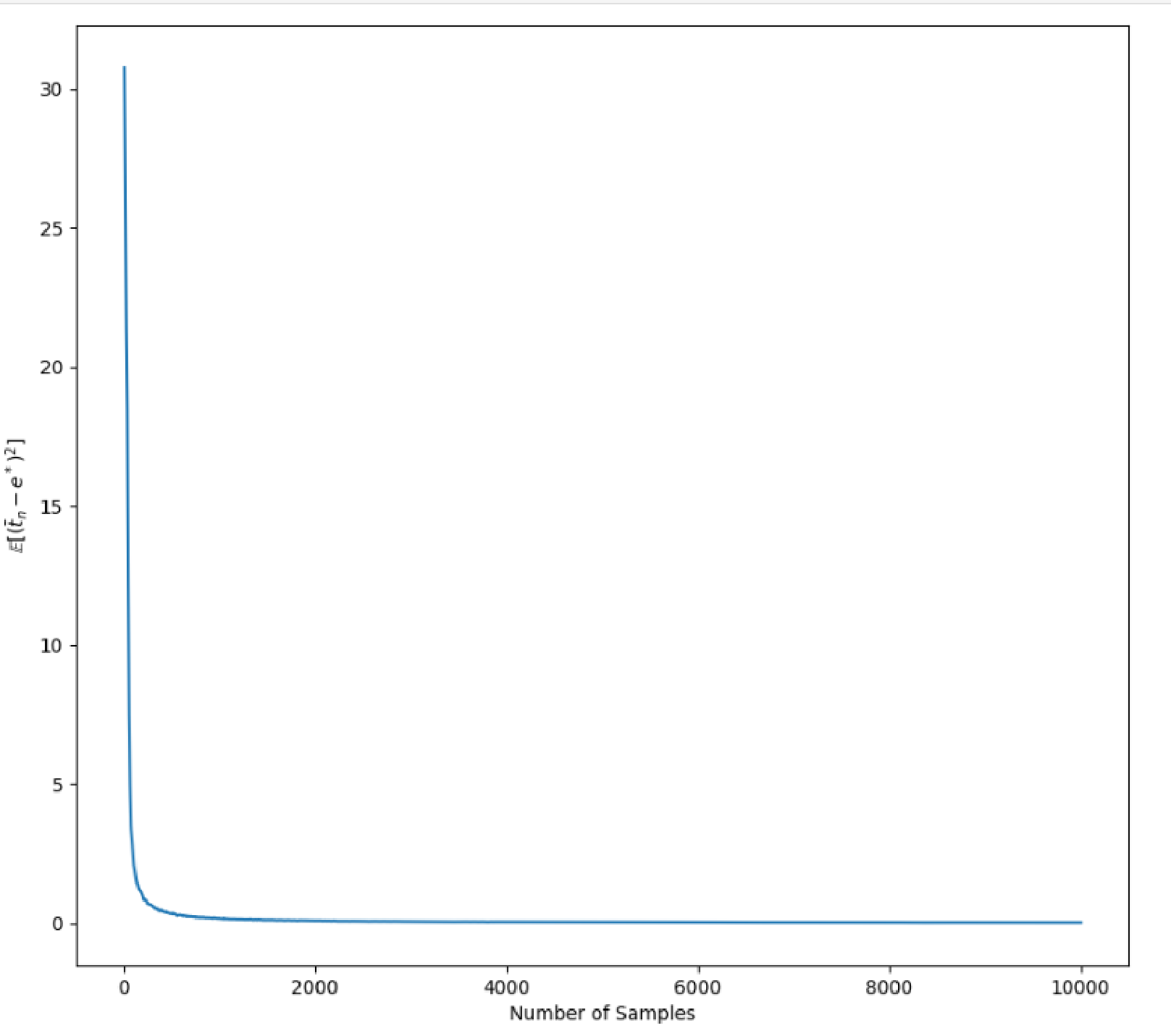}
        \caption{$\mathbb{E}[(\bar{t}_{k}-e^{*})^2]$}
        \label{fig:subfig5}
    \end{subfigure}
    \begin{subfigure}{0.45\textwidth}
        \centering
        \includegraphics[height=1.7in]{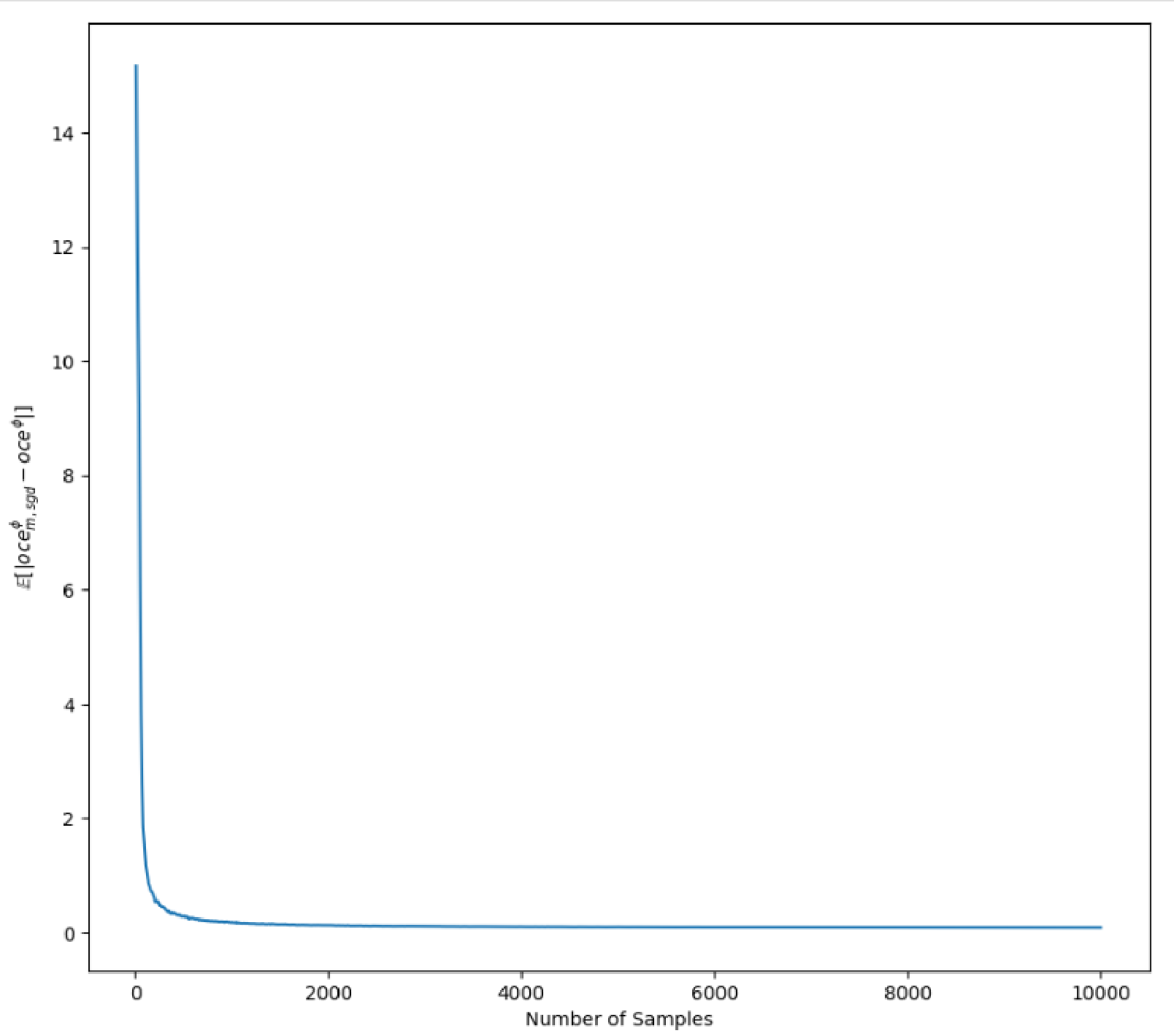}
        \caption{$\mathbb{E}[|oce^{\phi}_{k,\text{sgd}}-oce^{\phi}(L)|]$}
        \label{fig:subfig6}
    \end{subfigure}
    \caption{: Errors in estimation of OCE risk and its minimizer, under the credit risk model. \eqref{eq:sto-approx}, while OCE risk is estimated using \eqref{eq:oce-sgd}. The step size $\gamma_j= \frac{100}{j^{\alpha}}$ and $t_{0} = 1$. The results are averages over $1000$ independent replications.}
    \label{fig:cred_risk}
\end{figure}


\section{Conclusions and future work}
\label{sec:conclusions}
        We addressed the problem of OCE risk estimation from i.i.d. samples of the underlying loss distribution. We first considered an sample average OCE risk estimator, and derive a mean-squared error bound. We also derived a concentration bound for the sample average estimator when the underlying loss distribution is sub-Gaussian, and the disutility function is strongly convex and smooth. This concentration bound is useful in OCE risk-aware bandit applications. Finally, we also considered a stochastic root-finding based OCE risk estimator, and derived its finite sample guarantees.

        For future work, OCE risk estimation with Markovian samples remains unaddressed. OCE risk optimization in a risk-sensitive reinforcement learning framework is another interesting future research direction.

\bibliographystyle{plain}
\bibliography{refs_oce}


\appendix

\section{Proof of \Cref{thm:mse-estar}}

\begin{proof}
 \label{proof:thm:mse-estar}
   Using definitions \eqref{eq:phi-diff} and \eqref{eq:phin-diff}, we have

  \begin{equation}
    \frac{1}{n} \sum_{i=1}^{n} (\phi'(X_i - e^{*}) - \phi'(X_i - \hat{e}_{n}) ) = \frac{1}{n} \sum_{i=1}^{n} \Bigl( \phi'(X_i - e^{*}) -  E[\phi'(X_i-e^{*})] \Bigr)
  \end{equation}

We consider two cases for the analysis.

\textbf{Case 1:} ($\hat{e}_{n} > e^{*}$)

\[
\phi'(X_i - e^{*}) - \phi'(X_i - \hat{e}_{n}) \geq \mu(\hat{e}_{n} - e^{*}) \quad \text{(w.p. 1)}
\]
$$ \Rightarrow \frac{1}{n} \sum_{i=1}^{n} \Bigl( \phi'(X_i - e^{*}) -  E[\phi'(X_i-e^{*})] \Bigr) \geq \mu(\hat{e}_{n} - e^{*}), $$  
where we used the fact that $\phi$ is $\mu$-strongly convex.

\textbf{Case 2:} ($\hat{e}_{n} < e^{*}$)
\[
\phi'(X_i - e^{*}) - \phi'(X_i - \hat{e}_{n}) \geq \mu(\hat{e}_{n} - e^{*}) \quad \text{(w.p. 1)}
\]
    $$ \Rightarrow \frac{1}{n} \sum_{i=1}^{n} \Bigl(   E[\phi'(X_i-e^{*})] - \phi'(X_i - e^{*}) \Bigr) \geq \mu(  e^{*} - \hat{e}_{n}) $$ 
Thus, 
 \begin{equation}
    |e^{*} - \hat{e}_{n}| \leq \Bigl| \frac{1}{n\mu} \sum_{i=1}^{n} \Bigl(  \phi'(X_i - e^{*}) - E[\phi'(X_i-e^{*})]  \Bigr) \Bigr| 
    \label{eq:e*-en}
  \end{equation} 

Let $Z_{i} = \phi'(X_i - e^{*}) - \mathbb{E}[\phi'(X_i-e^{*})]$. Note that $\mathbb{E}[Z_{i}] = 0$. Squaring \ref{eq:e*-en} and taking expectations, we obtain
\begin{align*}
 \mathbb{E}[(e^{*} - \hat{e}_{n})^2] &\leq \frac{1}{n^2\mu^2} \mathbb{E}\bigl[(\sum_{i=1}^{n} Z_{i})^2\bigr] \\ &\leq \frac{1}{n^2\mu^2} \mathbb{E}\bigl[\sum_{i=1}^{n} Z_{i}^2 + \sum_{i,j=1}^{n}Z_{i}Z_{j}\bigr] \\ &\leq \frac{1}{n^2\mu^2} \sum_{i=1}^{n} \mathbb{E}[Z_{i}^2] \quad \text{(cross terms vanish since iid)} \\ &\leq \frac{1}{n^2\mu^2} \sum_{i=1}^{n} \mathbb{E}[(\phi'(X_i - e^{*}))^2] + (\mathbb{E}[\phi'(X_i-e^{*})])^2 - 2(\mathbb{E}[\phi'(X_i-e^{*})])^2 \\ &\leq \frac{1}{n^2\mu^2} \sum_{i=1}^{n} (\mathbb{E}[(\phi'(X_i - e^{*}))^2] - 1) \quad (\EE{\phi'(X-e^{*})} = 1)\numberthis \label{eq:s11}
\end{align*}

Using $\phi'(0) = 1$ and $L$-Lipschitzness of $\phi'$, we have 
\[|\phi'(X_{i}-e^{*}) - 1| \leq L|X-e^{*}|.\] 
Squaring on both sides above and taking expectations, we obtain
\begin{align}
\mathbb{E}[(\phi'(X-e^{*}))^2] &\leq (Le^{*})^2 + L^2\mathbb{E}[X^2] - 2L^2e^{*}\mathbb{E}[X] + 1.
\end{align}
The main claim follows by substituting the bound above in \eqref{eq:s11}.
\end{proof}

\section{Proof of \Cref{thm:estar-conc}}
  \label{proof:thm:estar-conc}
  
   For establishing the bound in \Cref{thm:estar-conc}, we require the following result, which shows that a Lipschitz function of a sub-Gaussian r.v. is sub-Gaussian.

\begin{lemma}
\label{lem:lipschitz-subg} 
Let $X$ $\in \mathbb{R}$ be sub-Gaussian with parameter $\sigma$. Then if $f : \mathbb{R} \to \mathbb{R}$ is $L$-Lipschitz, $f(X)$ is sub-Gaussian with parameter $2L\sigma$. 
\end{lemma} 

\begin{proof}
By Lemma 54 of \cite{prashanth2022wasserstein}, if $X$ is a sub-Gaussian r.v. with parameter $\sigma$, then
\begin{equation}
\label{eq:subg_moments}
(\mathbb{E} |X|^{k})^{\frac{1}{k}} \leq 2\sigma \sqrt{k}, \forall k \geq 1
\end{equation}

Following a proof technique analogous to that of the proof of Proposition 2.5.2 in \cite{vershynin2018high}, we have
\begin{align*}
\mathbb{E}[e^{\lambda^{2}X^{2}}]  &= \mathbb{E} \left[1 + \sum_{p=1}^{\infty} \frac{(\lambda^{2}X^{2})^{p}}{p!}\right] \\
&= 1 + \sum_{p=1}^{\infty} \frac{\lambda^{2p}\mathbb{E}[X^{2p}]}{p!} \\
&\leq \sum_{p=0}^{\infty} \lambda^{2p} (2\sigma)^{2p}2^{p}e^{p} \quad \text{(\eqref{eq:subg_moments} and } p! \geq \left(\frac{p}{e}\right)^p \text{)} \\
&\leq \sum_{p=0}^{\infty} (8e\sigma^2\lambda^2)^{p} \\ &= \frac{1}{1-8e\sigma^2\lambda^2} \quad \text{(if } |\lambda| \leq \frac{1}{\sqrt{8e}\sigma} \text{)} \\ 
&\leq e^{16e\sigma^2\lambda^2} \quad \text{(if } |\lambda| \leq \frac{1}{4\sqrt{e}\sigma} \text{)},
\numberthis\label{eq:X2_to_e}
\end{align*}
where the last step follows from the fact that $\frac{1}{1-x} \leq e^{2x} \quad \forall x \in [0,1/2]$.

Setting $\lambda = \frac{\sqrt{ln2}}{4\sqrt{e}\sigma}$, we obtain
\begin{equation}
\label{eq:X_K3}
\mathbb{E}\left[e^{X^2/K_{3}^{2}}\right] \leq 2 \quad \text{with } K_3^{2} = \frac{16e\sigma^2}{\ln 2}.
\end{equation}

Employing the proof technique of \url{https://mathoverflow.net/questions/442500/lipschitz-function-of-subgaussian-random-variable}, we have
\begin{align*}
\mathbb{E}\left[ e^{\lambda^{2}(f(X) - \mathbb{E}[f(X)])^{2}} \right] &= \mathbb{E}\left[ e^{\lambda^{2}(f(X) - \mathbb{E}[f(X')])^{2}} \right] \quad \text{($X'$ is an independent copy of $X$)} \\
&\leq \mathbb{E}\left[ e^{\lambda^{2}(f(X) - f(X'))^{2}} \right] \quad \text{(Jensen's inequality)} \\
&\leq \mathbb{E}\left[ e^{\lambda^{2}L^{2}(X - X')^{2}} \right] \quad \text{(Lipschitzness of $f$)} \\
&\leq \mathbb{E}\left[ e^{\lambda^{2}L^{2}(2X^{2} + 2X'^{2})} \right]
\\
&\leq \left(\mathbb{E}\left[ e^{2\lambda^{2}L^{2}X^{2}} \right]\right)^{2} \\
&\leq \mathbb{E}\left[ e^{4\lambda^{2}L^{2}X^{2}} \right] \quad \text{(Jensen's inequality)}
\numberthis\label{eq:fX&X}
\end{align*}
Setting $\lambda = \frac{\sqrt{ln2}}{8\sqrt{e}\sigma L}$, we obtain 
\begin{equation}
\mathbb{E}\left[e^{(f(X)-\mathbb{E}[f(X)])/K_{3}'^{2}}\right] \leq \mathbb{E}\left[e^{X^2/K_{3}^{2}}\right] \leq 2 \quad \text{with } K_{3}'^{2} = \frac{64e\sigma^2L^{2}}{\ln 2}.
\label{eq:fXleq2}
\end{equation}
We can now infer that $f(X) - \mathbb{E}[f(X)]$ is $4L^2\sigma^2$-sub-Gaussian by using Proposition 2.5.2 of \cite{vershynin2018high}.
\end{proof}

Next, we prove \Cref{thm:estar-conc} by using the result in the lemma above.
\begin{proof}\textbf{\textit{(\Cref{thm:estar-conc})}}
 \ \\ 
Recall from \eqref{eq:e*-en}, we have
 \begin{equation}
    |e^{*} - \hat{e}_{n}| \leq \Bigl| \frac{1}{n\mu} \sum_{i=1}^{n} \Bigl(  \phi'(X_i - e^{*}) - \EE{\phi'(X_i-e^{*})}  \Bigr) \Bigr|. \label{eq:sm12}
  \end{equation} 

Using Lemma \ref{lem:lipschitz-subg}, $\phi'(X_i - e^{*}) - \EE{\phi'(X_i-e^{*})}$ is $4L^{2}\sigma^{2}$ sub-Gaussian $\forall$ $i$. The main claim now follows by applying Hoeffding's inequality (cf. Proposition 2.1 of \cite{wainwright2019high}) to the RHS of \eqref{eq:sm12}.
\end{proof}

\section{Proof of Theorem \ref{thm:mse-oce}}

Using smoothness of $\phi$, we derive two useful bounds in the lemmas below. These bounds would be used subsequently int he proof of Theorem \ref{thm:mse-oce}.

\begin{lemma}
Under conditions of \Cref{thm:mse-oce}, we have
\begin{equation}
-L(e^{*}-\hat{e}_{n})^{2} + (e^{*}-\hat{e}_{n}) \leq (e^{*}-\hat{e}_{n}) \frac{\sum_{i=1}^{n} \phi'(X_{i}-e^{*})}{n} \leq L(e^{*}-\hat{e}_{n})^{2} + (e^{*}-\hat{e}_{n}).
\label{lem:1}
\end{equation}
\end{lemma}

\begin{proof}
By Lipschitzness of $\phi'$, 

\begin{equation}
    -L|\hat{e}_{n}-e^{*}| + \phi'(X_{i}-\hat{e}_{n}) \leq \phi'(X_{i} - e^{*}) \leq L|\hat{e}_{n}-e^{*}| + \phi'(X_{i}-\hat{e}_{n}) 
\end{equation}

\begin{equation}
   \Rightarrow -L|\hat{e}_{n}-e^{*}| + \frac{\sum_{i=1}^{n}\phi'(X_{i}-\hat{e}_{n})}{n} \leq \frac{\sum_{i=1}^{n}\phi'(X_{i} - e^{*})}{n} \leq L|\hat{e}_{n}-e^{*}| + \frac{\sum_{i=1}^{n}\phi'(X_{i}-\hat{e}_{n})}{n}
\end{equation}

\begin{equation}
   \Rightarrow -L|\hat{e}_{n}-e^{*}| + 1 \leq \frac{\sum_{i=1}^{n}\phi'(X_{i} - e^{*})}{n} \leq L|\hat{e}_{n}-e^{*}| + 1
\end{equation}

Now take 2 cases: $\hat{e}_{n} > e^{*}$ or vice-versa. According to the sign of $(\hat{e}_{n}-e^{*})$ take the relevant inequality, either case will yield:

\begin{equation}
-L(e^{*}-\hat{e}_{n})^{2} + (e^{*}-\hat{e}_{n}) \leq (e^{*}-\hat{e}_{n}) \frac{\sum_{i=1}^{n} \phi'(X_{i}-e^{*})}{n} \leq L(e^{*}-\hat{e}_{n})^{2} + (e^{*}-\hat{e}_{n})
\end{equation}

\end{proof}

\begin{lemma}
Under conditions of \Cref{thm:mse-oce}, we have
\begin{equation}
\frac{-3L\Bigl(e^{*} - \hat{e}_{n} \Bigr)^{2}}{2} + (e^{*} - \hat{e}_{n}) \leq \frac{\sum_{i=1}^{n} \phi(X_{i}-\hat{e}_{n})-\phi(X_{i}-e^{*}) }{n} \leq \frac{3L\Bigl(e^{*} - \hat{e}_{n} \Bigr)^{2}}{2} + (e^{*} - \hat{e}_{n})
\label{lem:2}
\end{equation}
\end{lemma}

\begin{proof}
By L-smoothness of $\phi$,
\begin{align}
   \frac{-L\Bigl(e^{*} - \hat{e}_{n} \Bigr)^{2}}{2} + (e^{*}-\hat{e}_{n})\phi'(X_{i}-e^{*}) & \leq \phi(X_{i}-\hat{e}_{n})-\phi(X_{i}-e^{*}) \\ & \leq \frac{L\Bigl(e^{*} - \hat{e}_{n} \Bigr)^{2}}{2} + (e^{*}-\hat{e}_{n})\phi'(X_{i}-e^{*})
\end{align}

\begin{align}
\frac{-L}{2}(e^{*} - \hat{e}_{n} )^{2} + (e^{*}-\hat{e}_{n}) \frac{\sum_{i=1}^{n} \phi'(X_{i}-e^{*})}{n} & \leq \frac{\sum_{i=1}^{n} \phi(X_{i}-\hat{e}_{n})-\phi(X_{i}-e^{*}) }{n} \\ & \leq \frac{L(e^{*} - \hat{e}_{n} )^{2}}{2} + (e^{*}-\hat{e}_{n}) \frac{\sum_{i=1}^{n} \phi'(X_{i}-e^{*})}{n}
\end{align}

And now just substitute the results obtained from  Lemma \ref{lem:1}.

\end{proof}
We now prove Theorem \ref{thm:mse-oce}.
\begin{proof}
\label{proof:thm:mse-oce}
Using the definition of $\ocen$, we have
\begin{equation}
   \Rightarrow \ocen = e^{*} + \sum_{i=1}^{n} \frac{\phi(X_{i} - e^{*})}{n} +  (\hat{e}_{n} - e^{*}) + \sum_{i=1}^{n} \frac{\phi(X_{i} - \hat{e}_{n}) - \phi(X_{i} - e^{*}) }{n}.
   \label{eq:ocen_def_diff}
\end{equation}

Using the results of Lemma \ref{lem:1} and Lemma \ref{lem:2}, we obtain
\begin{equation}
\frac{-3L\Bigl(e^{*} - \hat{e}_{n} \Bigr)^{2}}{2} + e^{*} + \frac{\sum_{i=1}^{n} \phi(X_{i}-e^{*})}{n} \leq \ocen \leq \frac{3L\Bigl(e^{*} - \hat{e}_{n} \Bigr)^{2}}{2} + e^{*} + \frac{\sum_{i=1}^{n} \phi(X_{i}-e^{*})}{n}.
\label{eq:oce_n_bound}
\end{equation}
Using $
\oce(X) = e^{*} + E\Bigl[ \phi(X - e^{*}) \Bigr]$, we obtain
\begin{equation}
\begin{split}
-\frac{3L\Bigl(e^{*} - \hat{e}_{n} \Bigr)^{2}}{2} &+  \frac{ \sum_{i=1}^{n} \phi(X_{i} - e^{*}) - E [ \phi(X_{i} - e^{*}) ] }{n} \\
&\leq \ocen - \oce(X) \\
&\leq \frac{3L\Bigl(e^{*} - \hat{e}_{n} \Bigr)^{2}}{2} + \frac{ \sum_{i=1}^{n} \phi(X_{i} - e^{*}) - E [ \phi(X_{i} - e^{*}) ] }{n} 
\end{split}
\label{eq:oce_broken_fin}
\end{equation}.
Recall that from \Cref{eq:e*-en}, we have
\begin{equation}
    |e^{*} - \hat{e}_{n}| \leq \Bigl| \frac{1}{n\mu} \sum_{i=1}^{n} \Bigl(  \phi'(X_i - e^{*}) - \EE{\phi'(X_i-e^{*})}  \Bigr) \Bigr|. 
\end{equation} 

Define $Z_{i} = \phi(X_{i}-e^{*}) - \mathbb{E}[\phi(X_{i}-e^{*})]$ and $Y_{i} = \phi'(X_{i}-e^{*}) - \mathbb{E}[\phi'(X_{i}-e^{*})] = \phi'(X_{i}-e^{*}) - 1 $. Note that $\mathbb{E}[Z_{i}] = 0$ and $\mathbb{E}[Y_{i}] = 0$. Using the fact that $(a+b)^2 \leq 2(a^2 + b^2)$, we obtain
\begin{align*}
 \mathbb{E}\Bigl[(\ocen - \oce(X))^2\Bigr] &\leq 2 \Bigl( \frac{1}{n^2} \EE{\Bigl(\sum_{i=1}^{n}Z_{i}\Bigr)^2} + \frac{9L^2}{4n^4\mu^4}\EE{\Bigl(\sum_{i=1}^{n}Y_{i}\Bigr)^4} \Bigr) \\ &\leq 2 \Bigl( \frac{1}{n^2}\sum_{i=1}^{n}\EE{Z_{i}^2} + \frac{9L^2}{4n^4\mu^4}\EE{\Bigl(\sum_{i=1}^{n}Y_{i}\Bigr)^4} \Bigr) \\ &\leq 2 \Bigl( \frac{1}{n^2}\sum_{i=1}^{n}\EE{Z_{i}^2} + \frac{9L^2}{4n^4\mu^4}\EE{\Bigl(\sum_{i,j,k,l=1}^{n}Y_{i}Y_{j}Y_{k}Y_{l}\Bigr)} \Bigr) \\ &\leq \frac{2}{n}\EE{Z^2} + \frac{9L^2}{2n^4\mu^4} \bigl( n\EE{Y^4} + \binom{n}{2} \times \binom{4}{2} \times (\EE{Y^2})^2 \bigr) \\ &\leq \frac{2}{n}\var(\phi(X-e^{*})) + \frac{27L^2(\EE{Y^2})^2}{2n^2\mu^4} + \frac{9L^2\EE{Y^4}}{2n^3\mu^4},
\end{align*}
where $Z = \phi(X-e^{*}) - \mathbb{E}[\phi(X-e^{*})]$ and $Y = \phi'(X-e^{*}) - \mathbb{E}[\phi'(X-e^{*})] = \phi'(X-e^{*}) - 1$.

Notice that $\EE{Z^2}=\var(\phi(X-e^{*}))$, 
  $\EE{Y^2}=\var(\phi'(X-e^{*})))$, and $\EE{Y^4} < \mathbb{E}[(\phi'(X-e^{*}))^4]$. Each of these expectations are finite and the main claim follows.
\end{proof}

\section{Proof of Lemma \ref{lem:lsmooth-subg}}

\begin{proof}
 \label{proof:lem:lsmooth-subg}
   Consider the r.v. $Z = \phi(X-e^{*}) - \mathbb{E}[\phi(X-e^{*})]$. Using $L$-smoothness of $\phi$, we know
\begin{align*}
    |\phi(X-e^{*}) - (\phi(0) + \phi'(0)(X-e^{*})) | &\leq L(X-e^{*})^2/2 
\end{align*}
\begin{align*}
\Rightarrow |\phi(X-e^{*})-(X-e^{*})| \leq L(X-e^{*})^2/2 \quad \text{(} \phi(0) = 0, \phi'(0) = 1 {)}.
\end{align*}
The bound above can be rewritten as follows:
\begin{equation}
    \phi(X-e^{*}) \leq aX^2 + b|X| + c,
\end{equation}
where  $a = L/2 , b = |Le^{*}-1|,$ and $c = \Bigl|\frac{L(e^{*})^2}{2} - e^{*}\Bigr|$. Notice that $a,b,c$ are non-negative constants.

Now we employ the proof technique adapted from the proof of \cite[Lemma 1.12]{rigollet2023highdimensional}. Using the dominated convergence theorem, we have
\begin{align*}
    &\mathbb{E}[e^{\lambda(\phi(X-e^{*}) - \mathbb{E}[\phi(X-e^{*})])}] \\
    &\leq 1 + \sum_{k=2}^{\infty} \frac{\lambda^{k} \mathbb{E} [\Bigl(\phi(X-e^{*}) - \mathbb{E}[\phi(X-e^{*})]\Bigr)^{k}]}{k!} \\
    &\leq 1 + \sum_{k=2}^{\infty} \frac{ \lambda^{k} 2^{k-1} \Bigl( \mathbb{E} [(\phi(X-e^{*}))^{k}] + (\mathbb{E}[\phi(X-e^{*})])^{k} \Bigr) }{k!} \quad \text{(Jensen's inequality)} \\
    &\leq 1 + \sum_{k=2}^{\infty} \frac{ \lambda^{k} 2^{k} \mathbb{E} [|\phi(X-e^{*})|^{k}]}{k!} \quad \text{(Jensen's inequality and } E[X] \leq E[|X|]) \\
    &\leq 1 + \sum_{k=2}^{\infty} \frac{\lambda^{k} 2^{k} \mathbb{E}[(aX^{2} + b|X| + c)^{k}]}{k!} \\
    &\leq 1 + \sum_{k=2}^{\infty} \frac{ \lambda^{k}2^{k}3^{k-1} \mathbb{E}[a^{k}X^{2k} + b^k|X|^{k}+c^{k}]}{k!} \quad \text{(Jensen's inequality)} \\
    &\leq 1 + \sum_{k=2}^{\infty} \frac{ \lambda^{k}6^{k}\mathbb{E}[a^{k}X^{2k} + b^k|X|^{k}+c^{k}]}{k!} \\
    & \leq 1 + I_{1} + I_{2} + I_{3}.\numberthis\label{eq:sq12}
\end{align*}
By part (II) of Theorem 2.2 in \cite{wainwright2019high}, we have that a r.v. $X$ is sub-exponential if there is a positive number $c_0 > 0$ such that $\mathbb{E}[e^{\lambda(X - \mathbb{E}[X])}] < \infty$ for all $|\lambda| \leq c_0$. 

We shall bound each of $I_{1},I_{2}$ and $I_{3}$ to establish a exponential moment bound for the r.v. $\phi(X-e^{*}) - \mathbb{E}[\phi(X-e^{*})]$.

Since $X$ is a sub-Gaussian r.v. with parameter $\sigma$, we have
\begin{equation}
    \EE{|X|^{k}} \leq (2\sigma^2)^{k/2}k\Gamma(k/2) \quad \forall k \geq 1.
\end{equation}
We now bound $I_1$ as follows:
\begin{align*}
    I_{1} & = \sum_{k=2}^{\infty} \frac{(6\lambda a)^{k}\mathbb{E}[X^{2k}]}{k!} \\
    & \leq \sum_{k=2}^{\infty} \frac{(6\lambda a)^{k}(2\sigma^{2})^{k}(2k)\Gamma(k)}{k!} \\
    & \leq 2\sum_{k=2}^{\infty} (12a\sigma^{2}\lambda)^{k} \\
    & \leq 2\sum_{k=2}^{\infty} \frac{1}{2^{k}} \quad \text{if } 12a\sigma^{2}\lambda \leq \frac{1}{2} \\
    & \leq 2 \quad \text{if } |\lambda| \leq \frac{1}{24a\sigma^{2}}.
\end{align*}
We now bound $I_1$ as follows:
\begin{align*}
    I_{2} & = \sum_{k=2}^{\infty} \frac{(6b\lambda)^{k}\mathbb{E}[|X|^{k}]}{k!} \\
    & \leq \sum_{q=1}^{\infty} \frac{(6b\lambda)^{2q+1}\mathbb{E}[|X|^{2q+1}]}{(2q+1)!} + \sum_{q=1}^{\infty} \frac{(6b\lambda)^{2q}\mathbb{E}[X^{2q}]}{(2q)!}  \\
    & \leq \sum_{q=1}^{\infty} \frac{ (6b\lambda)^{2q+1} (2\sigma^{2})^{q + \frac{1}{2}} (2q+1) \Gamma(q + \frac{1}{2})}{(2q+1)!} + \sum_{q=1}^{\infty} \frac{ (6b\lambda)^{2q} (2\sigma^{2})^{q} (2q) \Gamma(q)}{(2q)!} \\
    & \leq \sum_{q=1}^{\infty} \frac{(6\sqrt{2}b\sigma\lambda)^{2q+1}(2q+1)(2q)!\sqrt{\pi}}{(2q+1)!4^{q}q!} + 2 \sum_{q=1}^{\infty} \frac{(72b^2\sigma^2\lambda^2)^{q}q!}{(2q)!} \\
    & \leq (6\sqrt{2\pi}b\sigma\lambda) \sum_{q=1}^{\infty} \frac{(72b^2\sigma^2\lambda^2)^{q}}{q!} + 2 \sum_{q=1}^{\infty} \frac{q!}{2^{q}(2q)!} \quad \text{(if } 72b^{2}\sigma^{2}\lambda^{2} \leq \frac{1}{2}) \\
    & \leq (6\sqrt{2\pi}b\sigma\lambda)(e^{72b^{2}\sigma^{2}\lambda^{2}}- 1) + 2 \sum_{q=1}^{\infty} \frac{q!}{2^{q}(2q)!} \quad \text{(if } 72b^{2}\sigma^{2}\lambda^{2} \leq \frac{1}{2}) \\
    & \leq \frac{\sqrt{\pi} (\sqrt{e} - 1)}{\sqrt{2}} + 2 \sum_{q=1}^{\infty} \frac{q!}{2^{q}(2q)!} \quad \text{(if } 72b^{2}\sigma^{2}\lambda^{2} \leq \frac{1}{2}) \\
    &< 1 + 2 \sum_{q=1}^{\infty} \frac{q!}{2^{q}(2q)!} \quad \text{(if } 72b^{2}\sigma^{2}\lambda^{2} \leq \frac{1}{2}).
\end{align*}

To bound the second term above, consider $T_{q} = \frac{q!}{2^{q}(2q)!}$. $\frac{T{q+1}}{T_{q}} = \frac{1}{4(2q+1)}$. Then

\begin{align*}
    2\sum_{q=1}^{\infty} T_{q} & = 2(T_{1} + T_{2} +  T_{3} + ...)  \\
    & \leq 2(T_{1} + \frac{T_{1}}{4 \cdot 3} + \frac{T_{1}}{4^{2}\cdot 3 \cdot 5} + \frac{T_{1}}{4^{3}\cdot 3 \cdot 5 \cdot 7} + ...) \\
    & 2(\leq T_{1} + \frac{T_{1}}{4 \cdot 3} + \frac{T_{1}}{4^{3} \cdot 3} + \frac{T_{1}}{4^{5}\cdot 3} +  \frac{T_{1}}{4^{7}\cdot 3} + ...) \\
    & \leq 2T_{1} \Bigl (  1 + \frac{\frac{1}{4}}{3(1 - \frac{1}{16})} \Bigr) \\
    & \leq \frac{49}{90} < 1 \quad \text{(if } |\lambda| \leq \frac{1}{12b\sigma} ).
\end{align*}

We now bound $I_3$ as follows:
\begin{align*}
I_{3} & = \sum_{k=2}^{\infty} \frac{(6c\lambda)^{k}}{k!} \\
& = e^{6c\lambda} - 1 - 6c\lambda.
\end{align*}

Combining the bounds obtained for $I_{1}, I_{2}$ and $I_{3}$, we get
\begin{align}
\mathbb{E}[e^{\lambda(\phi(X-e^{*}) - \mathbb{E}[\phi(X-e^{*})])}] & \leq 1 + I_{1} + I_{2} + I_{3} \\
& \leq 1 + 2 + 1 + 1 + e^{6c\lambda} - 1 - 6c\lambda \label{eq:Lem3.2last} \\ 
& < \infty \quad \text{if } |\lambda| \leq \min\left(\frac{1}{24a\sigma^{2}}, \frac{1}{12b\sigma}\right).
\end{align}

We have showed that the zero mean r.v. $Y = \phi(X-e^{*}) - \mathbb{E}[\phi(X-e^{*})]$ is sub-exponential r.v. whenever $X$ is $\sigma^2$-sub-Gaussian. We also showed the existence of a $c_{0} = \min\left(\frac{1}{12L\sigma^{2}}, \frac{1}{12|Le^{*}-1|\sigma}\right)$ such that $\mathbb{E}[e^{\lambda( \phi(X-e^{*}) - \mathbb{E}[\phi(X-e^{*})])}] < \infty, \quad \forall |\lambda| \leq c_{0}$. \\

By the Chernoff Bound with $\lambda = \frac{c_{0}}{2}$, we have $P(Y \geq t) \leq \mathbb{E}[e^{\frac{c_{0}Y}{2}}]e^{-\frac{c_{0}t}{2}}$. Applying a similar argument to $-Y$, we can conclude that $P(|Y| \geq t) \leq c_{1}e^{-c_{2}t}$ with $c_{1} = \mathbb{E}[e^{\frac{c_{0}Y}{2}}] + \mathbb{E}[e^{\frac{-c_{0}Y}{2}}]$ and $c_{2} = \frac{c_{0}}{2}$. We can get an explicit bound for $c_1$ by setting $\lambda = \frac{c_0}{2}$ in \Cref{eq:Lem3.2last} to get $\mathbb{E}[e^{\frac{c_{0}Y}{2}}] \leq M$ where $M = 4 + e^{3|\frac{L(e^{*})^2}{2} - e^{*}|c_{0}} - 3|\frac{L(e^{*})^2}{2} - e^{*}|c_{0} $. By a parallel argument one can arrive at $\mathbb{E}[e^{\frac{-c_{0}Y}{2}}] \leq M$ so that $c_{1} \leq 2M = C_{1}$.

By a parallel argument to the proof of Lemma 55 of \cite{prashanth2022wasserstein}, we obtain

\begin{align*}
\mathbb{E}[|Y|^{k}] & = \int_{0}^{\infty} P(|Y|^k \geq u) \, du \\
& = \int_{0}^{\infty} P(|Y| \geq t ) k t^{k-1} \, dt \\
& \leq \int_{0}^{\infty} C_{1}e^{-c_{2}t} kt^{k-1} \, dt \\
& \leq  \int_{0}^{\infty} \frac{C_{1}k}{c_{2}^{k}} e^{-s} s^{k-1} \, ds \\
& = \frac{C_{1}k \Gamma(k)}{c_{2}^{k}} \\
& = \frac{1}{2} \Bigl(\frac{\sqrt{2C_{1}}}{c_{2}}\Bigr)^{2} k! (\frac{1}{c_{2}})^{k-2},
\end{align*}
which proves that $Y$ satisfies Bernstein's condition with parameters $\sigma^{2} = \frac{2C_{1}}{c_{2}^2}$ and $b = \frac{1}{c_{2}}$. 

We know from [p. 19, Chapter 2]\cite{wainwright2019high} that if $Y$ satisfies Bernstein's condition with parameters $(\sigma^{2},b)$ then it is also sub-exponential with parameters $(2\sigma^{2},2b)$. Thus we have that $\phi(X-e^{*}) - \mathbb{E}[\phi(X-e^{*})]$ is sub-exponential with parameters $\left(\frac{4C_{1}}{c_{2}^2},\frac{2}{c_{2}}\right)$, completing the proof. \\
\end{proof}

\section{Proof of \Cref{thm:oce-conc1}}

\begin{proof}
\label{proof:thm:oce-conc1}
   In Lemma \ref{lem:lsmooth-subg} we showed that $Y=\phi(X-e^{*}) - \mathbb{E}[\phi(X-e^{*})]$ is sub-exponential with parameters $(\frac{4C_{1}}{c_{2}^2},\frac{2}{c_{2}})$. 

Applying Bernstein's inequality to $Y$, we obtain

\begin{equation}
\mathbb P \Bigr( \Bigl| \phi(X - e^{*}) - \mathbb{E} [ \phi(X - e^{*}) ]\Bigr| \geq \epsilon \Bigl) \,\, \leq \,\, 2e^{ \frac{-c_{2}\epsilon^{2}}{2(2C_{1}+\epsilon)} }. \label{eq:bernstein} \\
\end{equation}

with $C_{1} = 2( 4 + e^{3|\phi(-e^{*})|c_{0}} - 3|\phi(-e^{*})|c_{0} )$, $c_{2} = \frac{c_{0}}{2}$ and $c_{0} = \min\left(\frac{1}{12L\sigma^{2}}, \frac{1}{12|\phi'(-e^{*})|\sigma}\right)$.

Using \Cref{eq:oce_broken_fin} derived during the proof of \Cref{thm:mse-oce}, we know
\begin{align}
    &\mathbb{P} \Bigl[ \left| \ocen - \oce(X) \right| > \epsilon \Bigr] \leq \nonumber\\
    &\quad \mathbb{P} \Bigl[ \left| \frac{1}{n} \sum_{i=1}^{n} \phi(X_{i} - e^{*}) - E(X_{i} - e^{*}) \right| > \epsilon/2 \Bigr] + P \Bigl[ \frac{3L}{2}(e^{*} - \hat{e}_{n} )^{2} > \epsilon/2 \Bigr] 
    \label{eq:oce_prob_broken}
\end{align}

The first term on the RHS of \Cref{eq:oce_prob_broken} is bounded using \Cref{eq:bernstein} as follows:
\begin{equation}
    \mathbb{P} \Bigl[ \left| \frac{1}{n} \sum_{i=1}^{n} \phi(X_{i} - e^{*}) - E(X_{i} - e^{*}) \right| > \frac{\epsilon}{2} \Bigr] \leq 2e^{ \frac{-c_{2}n\epsilon^{2}}{4(4c_{1}+\epsilon)}}
\end{equation}

The second term on the RHS of \Cref{eq:oce_prob_broken} can be bounded using the result derived in \Cref{thm:estar-conc}.

Combining the bounds on the aforementioned two terms, we obtain
\begin{equation}
\mathbb{P} \left[ \left| \text{oce}_{n}^{\phi} - \text{oce}^{\phi}(X) \right| > \epsilon \right] \leq 2e^{ \frac{-c_{2}n\epsilon^{2}}{4(4c_{1}+\epsilon)}} + 2e^{-\frac{\mu^{2}n\epsilon}{24L^{3}\sigma^{2}}}.
\end{equation}
Hence proved.\end{proof}

\section{Proof of \Cref{thm:bach-t_star}}

\begin{proof}
\label{proof:thm:bach-t_star}
   We rewrite below the update iteration \eqref{eq:sto-approx} in a form that is amenable to the application of a result from \cite{moulines2011long}
\begin{align}
t_j = t_{j-1} - \gamma_j f'_j(t_{j-1}) \quad \forall j \geq 1,\forall t_0 \in \mathbb{R},   \label{eq:stoapprox-rewrite}
\end{align}
where $X_{1},...,X_{n}$ are i.i.d r.v 's sampled during $m$ iterations of SGD. 

Define
\begin{align*}
    f(t) &= \mathbb{E} \left[ \phi(X-t) \right] + t, \,\,
    f'(t) = 1 - \mathbb{E} \left[ \phi'(X-t) \right], \\
    f_{j}(t) &= \phi(X_{j}-t) + t, \,\,
    f_{j}'(t) = 1 - \phi'(X_{j}-t).
\end{align*}
Then \Cref{thm:bach-t_star} can be proved by applying \cite[Thm. 3]{moulines2011long} which is subject to certain assumptions $(H_1, H_2’, H_3, H_4, H_6, H_7)$ of \cite{moulines2011long}. For the sake of completeness, we include them here, restating them for the scalar case. 

\begin{assumptionB}
    \label{ass:bach1}
    Let $(F_j)_{j>0}$ be an increasing family of $\sigma$-fields. $t_{0}$ is $F_0$-measurable, and for each $t \in \mathbb{R}$, the r.v. $f'_j(t)$ is square-integrable, $F_j$-measurable, and for all $t \in H$, $j > 1$, $E(f'_j(t) | F_{j-1}) = f'(t)$, with probability $1$.
\end{assumptionB}

\begin{assumptionB}
    \label{ass:bach2}
     For each $n > 1$, the function $f_n$ is almost surely convex, differentiable with Lipschitz-continuous gradient $f'_n$, with constant $L$, that is:
\[
\forall j > 1, \forall t_1, t_2 \in \mathbb{R}, \left| f'_j(t_1) - f'_j(t_2) \right| \leq L \left| t_1 - t_2 \right| \text{ w.p.1.}
\]
\end{assumptionB}

\begin{assumptionB}
    \label{ass:bach3}
     The function \( f \) is strongly convex, with convexity constant \( \mu > 0 \). That is, for all \( t_1, t_2 \in \mathbb{R} \), 
\[ f(t_1) \geq f(t_2) + f'(t_2)(t_1 - t_2)  + \frac{\mu}{2} (t_1 - t_2)^2. \]
\end{assumptionB}

\begin{assumptionB}
    \label{ass:bach4}
     There exists \( \sigma^2 \in \mathbb{R}^+ \) such that for all \( j > 1 \), \( \mathbb{E}[(f'_j(e^*)^2 \,|\, \mathcal{F}_{j-1}) \leq \sigma^2 \), w.p.1.
\end{assumptionB}

\begin{assumptionB}
    \label{ass:bach5}
     For each $n > 1$, the function $f_n$ is almost surely twice differentiable with Lipschitz-continuous second derivative $f_n''$, the Lipschitz constant being $M$. That is, for all $t_1, t_2 \in \mathbb{R}$ and for all $n > 1$,
\[ |f_n''(t_1) - f_n''(t_2)|  \leq M |t_1 - t_2|,\] (w.p 1).
\end{assumptionB}

\begin{assumptionB}
    \label{ass:bach6}
     There exists $\tau \in \mathbb{R}_+$, such that for each $j > 1$,
\[E(f_j'(e^{*})^4 | F_{j-1}) \leq \tau^4 \quad \text{almost surely}.\]
\end{assumptionB}

Now we show that the assumptions \Crefrange{ass:phi-stronglyconvex}{ass:DCT} and \ref{ass:b_combined} imply \ref{ass:bach1}--\ref{ass:bach6} hold for the update \eqref{eq:stoapprox-rewrite}. 

\begin{description}
    \item[\ref{ass:bach1}:] Holds by definition of $X_{j}$,$f_{j}'$ and $f$.
    \item[\ref{ass:bach2}:] Holds from the lipschitzness of $\phi'$.
    \item[\ref{ass:bach3}:] Consider $g(t) = \mathbb{E} \left[ \phi(X-t) \right]$. If this is strongly convex so is $g(t) + t$. By strong convexity of $\phi$,
    \begin{align*}
    &\phi(X-t_{2}) \geq \phi(X-t_{1}) + \phi'(X-t_{1})(t_{1}-t_{2}) + \frac{\mu}{2} (t_{1}-t_{2})^{2}\\
    &\Rightarrow \mathbb{E} \left[ \phi(X-t_{2}) \right] \geq \mathbb{E} \left[ \phi(X-t_{1}) \right] + \mathbb{E} \left[ \phi'(X-t_{1}) \right] (t_{1}-t_{2}) + \frac{\mu}{2} (t_{1}-t_{2})^{2}\\
    &\Rightarrow g(t_{2}) \geq g(t_{1}) + g'(t_{1})(t_{2}-t_{1}) + \frac{\mu}{2} (t_{1}-t_{2})^{2} 
    \end{align*}
    \item[\ref{ass:bach4}:] Holds trivially from \ref{ass:b_combined}.
    \item[\ref{ass:bach5}:] Follows from \ref{ass:b1}.
    \item[\ref{ass:bach6}:] This follows from \ref{ass:b_combined}, after noting that the samples are i.i.d. in our setting.
\end{description}
The proof of the theorem then follows by directly applying [Theorem 3]\cite{moulines2011non} in its scalar version to the $f,f',f_{j},f_{j}'$ defined earlier. 

Note that the first term in the bound from Theorem 3 of \cite{moulines2011long} is the dominant one. For our choice of $\alpha \in (0.5,1)$, it can be shown that all the others terms in the aforementioned bound are asymptotically smaller than the first one, and so they can all be represented in the form of $\frac{c_{i}}{\sqrt{m}}$. Also, the first term in scalar form is $\sigma^2f''(e^{*})^{-2} = \sigma^2 \EE{\phi''(X-e^{*})}^{-2} \leq \frac{\sigma^2}{\mu^2}$. 
\end{proof}

\section{Proof of \Cref{thm:exp_oce}}

\begin{proof}
\label{proof:thm:exp_oce}
Using the definitions of OCE risk and its estimator in \eqref{eq:oce-sgd}, we have
\begin{align}
&\ocesa (X) - \oce (X) \\ 
&= (\bar{t}_{m} - e^{*}) + \frac{\sum_{i=1}^{m} \phi(X_{i} - \bar{t}_{m})}{m} - \mathbb{E} \left[ \phi(X - e^{*}) \right] \nonumber \\
&= (\bar{t}_{m} - e^{*}) + \frac{\sum_{i=1}^{m} \phi(X_{i} - \bar{t}_{m}) - \sum_{i=1}^{m} \phi(X_{i}-e^{*})}{m} + \frac{\sum_{i=1}^{m} ( \phi(X_{i}-e^{*}) - \mathbb{E} \left[ \phi(X_{i} - e^{*}) \right] )}{m}.
\label{eq:oce-sgd-diff}
\end{align}
The second term on the RHS of \eqref{eq:oce-sgd-diff} can be bounded using $L$-smoothness as follows:
\begin{align*}
   &\frac{-L}{2} ( \bar{t}_{m} - e^{*})^{2} + \phi'(X-e^{*})(e^{*}-\bar{t}_{m})  \leq \\ &\qquad\qquad\qquad\phi(X - \bar{t}_{m}) - \phi(X - e^{*})  \leq \frac{L}{2} ( \bar{t}_{m} - e^{*})^{2} + \phi'(X-e^{*})(e^{*}-\bar{t}_{m}).
\end{align*}
Since the two-sided inequality above holds for any $X_i$, we obtain
\begin{align}
\frac{-L}{2} ( \bar{t}_{m} - e^{*})^{2} + (e^{*}-\bar{t}_{m}) \frac{\sum_{i=1}^{m} \phi'(X_{i}-e^{*})}{m} & \leq \\ \frac{\sum_{i=1}^{m} \phi(X_{i} - \bar{t}_{m}) - \sum_{i=1}^{m} \phi(X_{i}-e^{*})}{m} & \leq \\ \frac{L}{2} ( \bar{t}_{m} - e^{*})^{2} + (e^{*}-\bar{t}_{m}) \frac{\sum_{i=1}^{m} \phi'(X_{i}-e^{*})}{m}. \label{eq:e1}
\end{align}
Substituting \eqref{eq:e1} into \eqref{eq:oce-sgd-diff}, we obtain
\begin{align*}
 &\ocesa (X) - \oce (X)  \\
 &\qquad\leq \frac{L}{2}(\bar{t}_{m} - e^{*})^{2} + (\bar{t}_{m} - e^{*})\Bigl(1-\frac{\sum_{i=1}^{m} \phi'(X_{i}-e^{*})}{m}\Bigr) + \frac{\sum_{i=1}^{m} ( \phi(X_{i}-e^{*}) - \mathbb{E} \left[ \phi(X_{i} - e^{*}) \right] )}{m}.
\end{align*}
Using the triangle inequality, we have
\begin{align*}
 &\ocesa (X) - \oce (X)  \\
 &\quad\leq \frac{L}{2}(\bar{t}_{m} - e^{*})^{2} + \Bigl|(\bar{t}_{m} - e^{*})\Bigl(1-\frac{\sum_{i=1}^{m} \phi'(X_{i}-e^{*})}{m}\Bigr) \Bigr| + \Bigl| \frac{\sum_{i=1}^{m} ( \phi(X_{i}-e^{*}) - \mathbb{E} \left[ \phi(X_{i} - e^{*}) \right] )}{m} \Bigr|. \numberthis\label{eq:1}
\end{align*}
Using similar arguments, it can be shown that
\begin{align*}
 &\oce (X) - \ocesa (X) \\&\leq \frac{L}{2}(\bar{t}_{m} - e^{*})^{2} + (e^{*}-\bar{t}_{m})\Bigl(1-\frac{\sum_{i=1}^{m} \phi'(X_{i}-e^{*})}{m}\Bigr) -  \frac{\sum_{i=1}^{m} ( \phi(X_{i}-e^{*}) - \mathbb{E} \left[ \phi(X_{i} - e^{*}) \right] )}{m}.
\end{align*}
As before, using the triangle inequality, we obtain
\begin{align*}
 &\oce (X) - \ocesa (X) \leq \frac{L}{2}(\bar{t}_{m} - e^{*})^{2} \\
 &\qquad+ \Bigl|(\bar{t}_{m} - e^{*})\Bigl(1-\frac{\sum_{i=1}^{m} \phi'(X_{i}-e^{*})}{m}\Bigr) \Bigr| + \Bigl| \frac{\sum_{i=1}^{m} ( \phi(X_{i}-e^{*}) - \mathbb{E} \left[ \phi(X_{i} - e^{*}) \right] )}{m} \Bigr|. \numberthis\label{eq:2}
\end{align*}
It is evident from \eqref{eq:1} and \eqref{eq:2} that 
\begin{align*}
    |\ocesa (X) - \oce (X)| &\leq \frac{L}{2}(\bar{t}_{m} - e^{*})^{2} + \Bigl|(\bar{t}_{m} - e^{*})\Bigl(1-\frac{\sum_{i=1}^{m} \phi'(X_{i}-e^{*})}{m}\Bigr) \Bigr| \\
    &\qquad + \Bigl| \frac{\sum_{i=1}^{m} ( \phi(X_{i}-e^{*}) - \mathbb{E} \left[ \phi(X_{i} - e^{*}) \right] )}{m} \Bigr|,
\end{align*}
implying
\begin{align*}
    \EE{|\ocesa (X) - \oce (X)|} &\leq \frac{L}{2}\EE{(\bar{t}_{m} - e^{*})^{2}}  \\
    &\quad+  \EE{\Bigl|(\bar{t}_{m} - e^{*})\Bigl(1-\frac{\sum_{i=1}^{m} \phi'(X_{i}-e^{*})}{m}\Bigr) \Bigr| } \\ &\quad+ \EE{\Bigl| \frac{\sum_{i=1}^{m} ( \phi(X_{i}-e^{*}) - \ \mathbb{E} \left[ \phi(X_{i} - e^{*}) \right] )}{m} \Bigr|}. \numberthis\label{eq:3}
\end{align*}

The first term of  \Cref{eq:3} is bounded above by $\frac{L \K^2}{2m}$ from the result derived in \Cref{thm:bach-t_star}. To bound the second term, we invoke the Cauchy-Schwarz inequality and simplify as follows:
\begin{align*}
\EE{(\bar{t}_{m} - e^{*})\left(1-\frac{\sum_{i=1}^{m} \phi'(X_{i}-e^{*})}{m}\right)} \leq \left(\EE{(\bar{t}_{m} - e^{*})^2}\EE{\left(1-\frac{\sum_{i=1}^{m} \phi'(X_{i}-e^{*})}{m}\right)^2}\right)^{1/2}.
\end{align*}
Using \Cref{thm:bach-t_star}, the first term on the RHS is bounded above $\K^2/m$. To bound the second term, notice that
\begin{align*}
    \EE{\Bigl(1-\frac{\sum_{i=1}^{m} \phi'(X_{i}-e^{*})}{m}\Bigr)^2} &\leq \EE{1 - \frac{2\sum_{i=1}^{m} \phi'(X_{i}-e^{*})}{m} + \frac{(\sum_{i=1}^{m} \phi'(X_{i}-e^{*}))^2}{m^2} } \\ &\leq 1 - 2 + \frac{ m(\var{(\phi'(X-e^{*}))}+1) + (m^2 - m) }{m^2} \\ &\leq -1 + \frac{\var{(\phi'(X-e^{*}))}}{m} + 1 \\ &\leq \frac{\var{(\phi'(X-e^{*}))}}{m},
\end{align*}
since $\EE{\phi'(X-e^{*})} = 1$, and $\var{(\phi'(X-e^{*}))} = \EE{(\phi'(X-e^{*})^2} - 1$.

From the foregoing, we have
\begin{equation}
\EE{(\bar{t}_{m} - e^{*})\Bigl(1-\frac{\sum_{i=1}^{m} \phi'(X_{i}-e^{})}{m}\Bigr)} \leq \frac{\K\sqrt{\var{(\phi'(X-e^{*}))}}}{m}.
\end{equation}
We now bound the third term of \Cref{eq:3} using \Cref{ass:b_combined} and Jensen's inequality as follows:
\begin{align*}
&\mathbb{E}\left|\frac{\sum_{i=1}^{m} \left(\phi(X_{i}-e^{*}) - \mathbb{E}[\phi(X_{i} - e^{*})]\right)}{m}\right|\\
&\leq \frac{1}{m} \sqrt{\sum_{i=1}^{m} \EE{ (\phi(X_{i}-e^{*}) - \mathbb{E}[\phi(X_{i} - e^{*})])^2}} \leq \frac{\sqrt{\var{(\phi(X-e^{*}))}}}{\sqrt{m}}.
\end{align*}
Substituting the bounds on the three terms on the RHS of \eqref{eq:3}, we obtain
\begin{equation*}
\EE{|\ocesa(X) - \oce(X)|} \leq \frac{L \K^2}{2m} + \frac{\K\sqrt{\var{(\phi'(X-e^{*}))}}}{m} + \frac{\sqrt{\var{(\phi(X-e^{*}))}}}{\sqrt{m}}.
\end{equation*}
Hence proved.
\end{proof}

\section{Proof of \Cref{thm:oce-algo}}

\begin{proof}
\label{proof:thm:oce-algo}
The proof closely follows a similar claim for optimizing CVaR in a best arm identification framework, see Theorem 3.6 in \cite{prashanth2020concentration}.

Consider \Cref{eq:oce_prob_broken}. Since $\phi(X_{i} - e^{*}) - E(X_{i} - e^{*})$ is sub-exponential via \Cref{lem:lsmooth-subg}, invoking [Proposition 2.2]\cite{wainwright2019high}, the first term can be bounded as,

\[
P \Bigl[ | \frac{1}{n} \sum_{i=1}^{n} \phi(X_{i} - e^{*}) - E(X_{i} - e^{*}) | > \epsilon/2 \Bigr] \leq
\begin{cases}
    2e^{-\frac{n\epsilon^2c_2^2}{32c_1}} & \text{if } 0 \leq \epsilon \leq \frac{4c_1}{c_2} \\
    2e^{-\frac{c_2n\epsilon}{8}} & \text{if } \epsilon > \frac{4c_1}{c_2}
\end{cases}
\]

The second term is bounded using the result derived in \Cref{thm:estar-conc},

\begin{equation}
P \left( \frac{3L}{2} (e^{*} - \hat{e}_{n} )^{2} > \frac{\epsilon}{2} \right) \leq 2e^{-\frac{\mu^{2}n\epsilon}{24L^{3}\sigma^{2}}}
\end{equation}

Combine both the bounds to get an upper bound on the OCE-estimate as:

\begin{equation}
    P \left[ \left| \text{oce}_{n}^{\phi} - \text{oce}^{\phi}(X) \right| > \epsilon \right] \leq 
    \begin{cases}
    2e^{-\frac{n\epsilon^2c_2^2}{32C_1}} + 2e^{-\frac{\mu^{2}n\epsilon}{24L^{3}\sigma^{2}}} & \text{if } 0 \leq \epsilon \leq \frac{4C_1}{c_2} \\
    2e^{-\frac{c_2n\epsilon}{8}} + 2e^{-\frac{\mu^{2}n\epsilon}{24L^{3}\sigma^{2}}} & \text{if } \epsilon > \frac{4C_1}{c_2}
\end{cases}
\label{eq:oce_split}
\end{equation}

Letting $G = \min \left\{ \frac{c_{2}^2}{32C_{1}}, \frac{c_{2}}{8}, \frac{\mu^2}{24L^{3}\sigma^2} \right\}$, the bound can be written in a simplified form as:

\begin{align}
\mathbb{P}\left[\left| \text{oce}_{n}^{\phi} - \text{oce}^{\phi}(X)  \right|>\epsilon\right] \leq 4 e^{-n \min \left\{\epsilon, \epsilon^{2}\right\}G} 
\label{eq:oce_split_simplified}
\end{align}

If the OCE-SR algorithm eliminates the optimal arm in phase \( i \), it means that at least one of the arms in the set of the last \( i \) worst arms, denoted as \( \{[K],[K-1], \ldots,[K-i+1]\} \), must not have been eliminated in phase \( i \). Therefore, we conclude that:

\begin{align*}
& \mathbb{P}\left[J_{n} \neq i^{*}\right] \leq \sum_{k=1}^{K-1} \sum_{i=K+1-k}^{K} \mathbb{P}\left[\hat{oce}_{n_{k}}^{i^{*}} \geq \hat{oce}_{n_{k}}^{[i]}\right] \\
& =\sum_{k=1}^{K-1} \sum_{i=K+1-k}^{K} \mathbb{P}\left[\hat{oce}_{n_{k}}^{i^{*}}-oce_{\phi}^{i^{*}}-\hat{oce}_{n_{k}}^{[i]}+oce_{\phi}^{[i]} \geq oce_{\phi}^{[i]}-oce_{\phi}^{i^{*}}\right] \\
& \leq \sum_{k=1}^{K-1} \sum_{i=K+1-k}^{K} \mathbb{P}\left[\hat{oce}_{n_{k}}^{i^{*}}-oce_{\phi}^{i^{*}} \geq \frac{\Delta_{[i]}}{2}\right]+\sum_{k=1}^{K-1} \sum_{i=K+1-k}^{K} \mathbb{P}\left[oce_{\phi}^{[i]}-\hat{oce}_{n_{k}}^{[i]} \geq \frac{\Delta_{[i]}}{2}\right] 
\label{eq:oce_bandit_maineqn}
\end{align*}

We now bound the above terms individually as follows.

\begin{align}
& \sum_{k=1}^{K-1} \sum_{i=K+1-k}^{K} \mathbb{P}\left[oce_{\phi}^{[i]}-\hat{oce}_{n_{k}}^{[i]} \geq \frac{\Delta_{[i]}}{2}\right] \leq \sum_{k=1}^{K-1} \sum_{i=K+1-k}^{K} \mathbb{P}\left[\left|\hat{oce}_{n_{k}}^{[i]}-oce_{\phi}^{[i]}\right| \geq \frac{\Delta_{[i]}}{2}\right] \\
& \overset{(a)}{\leq} \sum_{k=1}^{K-1} \sum_{i=K+1-k}^{K-1} 4 e^{-n \min \left\{\frac{\Delta_{[i]}}{2} \frac{\Delta_{[i]}^{2}}{4}\right\} G_{[i]}} \\
& \leq \sum_{k=1}^{K-1} \sum_{i=K+1-k}^{K} 4e^{ -n\min \left\{\frac{\Delta_{[i]}}{2}, \frac{\Delta_{[i]}^{2}}{4}\right\} G_{\max }} \\
& \leq \sum_{k=1}^{K-1} 4k e^{-n \min \left\{\frac{\Delta_{[K+1-k]}}{2}, \frac{\Delta_{[K+1-k]}^{2}}{4}\right\} \times G_{\max }} 
\label{eq:oce_bandit_eqn}
\end{align}

where $(a)$ is due to \Cref{eq:oce_split} and \Cref{eq:oce_split_simplified} , and $G_{\max }=\max _{i} G_{i}$. Further, note that

$$
n \min \left\{\frac{\Delta_{[K+1-k]}}{2}, \frac{\Delta_{[K+1-k]}^{2}}{4}\right\} \geq \frac{n-K}{H \overline{\log } K}
$$

where $H=\max _{i \in\{1,2 \ldots, K\}} \frac{i}{\min \left\{\Delta_{[i]} / 2, \Delta_{[i]}^{2} / 4\right\}}$. By substituting the above in \Cref{eq:oce_bandit_eqn}, we obtain

\begin{align}
\sum_{k=1}^{K-1} \sum_{i=K+1-k}^{K} \mathbb{P}\left[oce_{\phi}^{[i]}-\hat{oce}_{n_{k}}^{[i]} \geq \frac{\Delta_{[i]}}{2}\right] \leq \sum_{k=1}^{K-1} 4 k e^{-\frac{(n-K)G_{\max }}{H \overline{\log } K}}
\label{eq:oce_bandit_1}
\end{align}

Similarly, it can be shown that

\begin{align}
\sum_{k=1}^{K-1} \sum_{i=K+1-k}^{K} \mathbb{P}\left[\hat{oce}_{n_{k}}^{i^{*}}-oce_{\phi}^{i^{*}}-\geq \frac{\Delta_{[i]}}{2}\right] \leq \sum_{k=1}^{K-1} 4 k e^{-\frac{(n-K)G_{\max }}{H \overline{\log } K}}.
\label{eq:oce_bandit_2}
\end{align}

The theorem follows by substituting \Cref{eq:oce_bandit_1} and \Cref{eq:oce_bandit_2} in \Cref{eq:oce_bandit_maineqn}.

\end{proof}

\end{document}